\documentclass{article}
\usepackage{arxivstyle}
\iclrfinalcopy
\usepackage[utf8]{inputenc}
\usepackage[T1]{fontenc}
\usepackage{array}
\usepackage{wrapfig}
\usepackage{hyperref}
\usepackage{url}
\usepackage{booktabs}
\usepackage{amsmath}
\usepackage{amssymb}
\usepackage{mathtools}
\usepackage{amsthm}
\usepackage{graphicx}
\usepackage{float}
\usepackage{algorithmic}
\usepackage{algorithm}
\usepackage{microtype}
\usepackage{xcolor}
\usepackage{adjustbox}
\usepackage{subcaption}
\mathtoolsset{showonlyrefs}
\usepackage[utf8]{inputenc}
\usepackage[T1]{fontenc}
\usepackage{hyperref}
\usepackage{url}
\usepackage{booktabs}
\usepackage{amsmath}
\usepackage{amssymb}
\usepackage{mathtools}
\usepackage{amsthm}
\usepackage{graphicx}
\usepackage{float}
\usepackage{algorithmic}
\usepackage{algorithm}
\usepackage{microtype}
\usepackage{xcolor}

\mathtoolsset{showonlyrefs}

\newtheorem{theorem}{Theorem}[section]
\newtheorem{proposition}[theorem]{Proposition}
\newtheorem{corollary}[theorem]{Corollary}
\theoremstyle{remark}
\newtheorem{remark}[theorem]{Remark}

\newcommand{\R}{\mathbb{R}}
\renewcommand{\d}{\,\mathrm{d}}
\newcommand{\W}{\operatorname{W}}
\newcommand{\mP}{\mathcal{P}}
\newcommand{\bmu}{\boldsymbol{\mu}}
\newcommand{\bnu}{\boldsymbol{\nu}}
\newcommand{\bx}{\boldsymbol{x}}
\newcommand{\by}{\boldsymbol{y}}
\newcommand{\pr}{\mathrm{proj}}
\newcommand{\KL}{\operatorname{KL}}
\newcommand{\GM}{\operatorname{GM}}
\DeclareMathOperator*{\argmin}{arg\,min}

\hypersetup{
  pdftitle={Alignment Matters Inside and Out in Equivariant Graph Flow Matching},
  pdfauthor={Anonymous Author(s)}
}

\definecolor{moritzgreen}{RGB}{46,125,50}

\renewcommand{\d}{\,\mathrm{d}}
\newcommand{\smallcite}[1]{{\tiny\citep{#1}}}

\title{Alignment Matters Inside and Out\\in Equivariant Graph Flow Matching}

\author{Anonymous Author(s)}

\hypersetup{
  pdftitle={Alignment Matters Inside and Out in Equivariant Graph Flow Matching},
  pdfauthor={M. Piening}
}

\title{Alignment Matters Inside and Out\\in Equivariant Graph Flow Matching}

\author{Moritz Piening \\
Institute of Computer Science \\
University of Göttingen \\
Göttingen, Germany \\
\texttt{moritz.piening@uni-goettingen.de}
\And
Christian Wald \\
Institut Camille Jordan \\
INSA Lyon \\
Villeurbanne, France \\
\texttt{christian.wald@insa-lyon.fr}
}

\begin{document}

\maketitle

\begin{abstract}
\scriptsize
Graphs are invariant under node permutations, motivating permutation-equivariant architectures in generative models. In flow matching, however, symmetry also affects the source--target coupling: one must choose both \emph{which graphs to pair} and \emph{which node representatives to align}. We study these two forms of alignment, termed outer and inner alignment, and their effect on equivariant graph flow matching. We connect inner alignment to transport on the graph quotient space, whose Euclidean quotient metric coincides with Gromov--Monge distance, and show that quotient couplings admit aligned representative lifts while symmetrization yields equivariant flow-matching minimizers, including for categorical endpoints. In practice, we compare random augmentation, permutation-blind minibatch optimal transport, approximate Gromov--Wasserstein alignment, and combinations of inner and outer alignment. Across continuous graph and molecular generation, alignment can substantially simplify trajectories and improve few-step generation, while its benefits depend on the alignment and computational budget. \textcolor{black}{Our results highlight that effective graph flow matching benefits from alignment both inside and across graph pairs.}
\end{abstract}
\vspace{-3.5mm}\section{Introduction}

\label{sec:introduction}

Flow matching learns a velocity field that transports a tractable source law
to data using simulation-free regression and ordinary
differential equation (ODE) sampling
\citep{albergo2023building_fm,lipman2023flowmatching_fm,liuflow_fm_recti}.
Its conditional paths are determined by a source--target coupling. Although all
such couplings have the same marginals at start and end time, transport-informed choices can produce
shorter, straighter trajectories and simplify numerical
integration \citep{chemseddine2025conditional,tong2023improving_minibatch,
pooladian2023multisample_couplings}. The correct notion of transport is less
obvious when data are defined only up to symmetry \citep{klein2023equivariant_flow_matching,kohler2020equivariant}.
\\
For graphs the symmetry is node relabelling. Here, we encode edge channels and
diagonal node features of an \(N\)-node graph by
\(E\in\R^{N\times N\times C}\). Because a graph has no distinguished vertex
order, a permutation \(\sigma\in\mathfrak S_N\) acts simultaneously on both
node indices according to
\[
\bigl(\rho_N^{\rm graph}(\sigma)E\bigr)_{ijc}
\coloneqq
E_{\sigma^{-1}(i)\sigma^{-1}(j)c},
\qquad
1 \leq i, j \leq N.
\]
The unordered graph is the orbit
\[
[E]\in Q\coloneqq\R^{N^2C}/G_N^{\rm graph},\qquad
G_N^{\rm graph}=\rho_N^{\rm graph}(\mathfrak S_N)\subset O(N^2C).
\]
Equivariant architectures employed in graph generative models respect this symmetry
\citep{eijkelboom2024variational,jo2022gdss,vignac2023digress}, but many models do not
by themselves choose which labelled representatives to connect during flow
matching. Each conditional bridge from \(E\) to an
unordered target \([F]\) may select some \(\sigma\cdot F\). Minimizing its length is
a \emph{Gromov--Monge} node-alignment problem, the hard-assignment counterpart of
Gromov--Wasserstein transport
\citep{memoli2011gromov,memoli2024comparison}.
\\
We connect this representative choice to equivariant flow matching on general
quotients \(\R^D/G\) \citep{klein2023equivariant_flow_matching,
kohler2020equivariant,song2023equivariant_flow_matching}. Our theory shows that
optimal quotient couplings lift without extra cost and that symmetrizing a
coupling between chosen representatives yields an equivariant minimizer of the flow-matching
 objective, including for categorical endpoint prediction. 
 For graphs, this yields an inner node-alignment problem corresponding to
Gromov--Monge distance. We approximate it using Gromov--Wasserstein solvers
and cheaper alignment heuristics \citep{memoli2011gromov,bauer2025zgw,piening2025novel},
and separately study outer minibatch assignment between graphs. This changes only the source--target coupling, leaving the flow-matching setup unchanged.

\vspace{-3mm}\paragraph{Relation to prior work}
Graph generators commonly enforce node-relabeling symmetry through equivariant
architectures, including molecular and discrete diffusion
\citep{hoogeboom2022edm,jo2022gdss,vignac2023digress} and variational flow
matching \citep{eijkelboom2024variational}. Recent graph flow methods additionally
optimize the source--target graph pairing using minibatch optimal transport
\citep{hou2026ggflow,wijesinghe2026flowette}. We complement these approaches by explicitly separating alignment inside and
across graph pairs: inner alignment selects representatives within each graph
orbit, while outer alignment assigns graphs within a minibatch. We refer to
Appendix~\ref{supp_sec:related_work} for a detailed comparison.
\\
Optimal-transport couplings can straighten Euclidean flow-matching paths
\citep{chemseddine2025conditional,pooladian2023multisample_couplings,tong2023improving_minibatch} and can be extended to geometries with invariances \citep{klein2023equivariant_flow_matching,kohler2020equivariant}.
In particular, if samples themselves carry an optimal-transport geometry, this leads to nested transport formulations, recently used for point-cloud generation
\citep{haviv2025wasserstein,piening2026generalized_wow_fm,piening2025slicing_wow}.
For graphs, node correspondences induce a Gromov--Monge ground cost, naturally relaxed by Gromov--Wasserstein methods for relational alignment
\citep{beier2025joint,chowdhury2021quantized,peyre2016gromov,vayer2020fused}.
Beyond particle permutations by linear
assignment \citep{klein2023equivariant_flow_matching,haviv2025wasserstein,hui2025notsooptimal_point_cloud_flow_matching,piening2026generalized_wow_fm}, the proposed graph matching relabels graph and node indices before interpolation.
Overall, this results in defining equivariant flow matching \citep{klein2023equivariant_flow_matching} over a Gromov-type graph geometry \citep{memoli2024comparison}. Particularly, 
our contributions are:
\begin{itemize}
    \item We show that optimal quotient couplings lift to Euclidean couplings
    with the same quadratic cost, whose interpolations project to
    constant-speed Wasserstein geodesics.
    
    \item We show that diagonal symmetrization preserves the flow-matching
    objective on equivariant fields and yields an equivariant minimizer,
    including for categorical endpoint prediction.
    
    \item For graphs, we separate alignment inside and across graph pairs:
    inner alignment selects node representatives, while outer alignment
    assigns graphs within a minibatch. We study Gromov--Wasserstein and
    finite-permutation approximations for the former and minibatch optimal
    transport for the latter.
    
    \item Across continuous and categorical graphs, we study how inner and
    outer alignment affect trajectory straightness and few-step generation,
    including their scaling with alignment compute. Better inner alignment
    yields substantial few-step gains, while outer alignment provides
    smaller, dataset-dependent improvements.
\end{itemize}

\vspace{-3.5mm}\section{Background on transport and flow matching}
In flow matching generative models, curves of probability measures are constructed that connect an easy-to-sample source measure with a target measure which is available only through samples \citep{lipman2023flowmatching_fm,wald2025flow}.
This construction is not limited to probability measures on Euclidean spaces \citep{chen2024flow_riemannian_flow_matching}.
After a short recap on Wasserstein geometry \citep{AmbrosioGigliSavare2005,Santambrogio2015,wald2025flow}, we will see how the same construction can be formulated on quotient spaces induced by orthogonal group actions.

\vspace{-3.5mm}\subsection{Wasserstein geometry and dynamic transport}
Let $(X,d_X)$ be a Polish metric space and $\mP_2(X)$ the set of Borel probability measures on $X$ with finite second moments, i.e.
$
\int_X d_X(x,x_0)^2 \,\d\mu(x) < \infty 
$
for some $x_0 \in X$.
The set $\mP_2(X)$ becomes a complete metric space with the \emph{Wasserstein distance} $\W_{2,X}$ \citep{Santambrogio2015,Villani2009oldandnew} which is given for any $\mu,\nu \in \mP_2(X)$ by
\begin{equation}
\label{eq:wasserstein}
\W_{2,X}^2(\mu,\nu)
\coloneqq
\min_{\pi \in \mathrm{c}_X(\mu,\nu)}
\int_{X \times X} d_X(x,x')^2 \,\d\pi(x,x').
\end{equation}
The \emph{couplings} or \emph{transport plans} are given by 
\[
\mathrm{c}_X(\mu,\nu)
=
\{ \pi \in \mP_2(X \times X): \mathrm{proj}^0_\sharp \pi = \mu,\mathrm{proj}^1_\sharp \pi = \nu\},
\]
with $\mathrm{proj}^i: X \times X \to X$, $(x_0,x_1) \mapsto x_i$. 
By $\mathrm{c}_{X}^{{\rm opt}}(\mu,\nu)$
we denote \emph{optimal} couplings, where the minimum in \eqref{eq:wasserstein} is attained.
Throughout, let $I=[0,1]$.
\\
For $X=\R^D$, let $(\mu_t)_{t\in I}$ be narrowly continuous and let
$v:I\times\R^D\to\R^D$ be Borel. The pair $(\mu_t,v_t)$ satisfies the
continuity equation if
\begin{equation}
\label{ce1}
\partial_t\mu_t+\operatorname{div}(v_t\mu_t)=0
\end{equation}
in the sense of distributions. If, in addition,
\[
\int_I\int_{\R^D}\|v_t(x)\|^2\,\d\mu_t(x)\,\d t<\infty,
\]
then $(\mu_t)_{t\in I}$ is absolutely continuous with respect to $\W_{2,\R^D}$.
Conversely, curves that are absolutely continuous with respect to
$\W_{2,\R^D}$ and have square-integrable speed admit a Borel field $v$
satisfying \eqref{ce1} and the displayed integrability condition
\citep[Chapter~8]{AmbrosioGigliSavare2005}. The definitions and bounds used
below are recalled in Appendix~\ref{app:measure-transport}. If $v$ is
sufficiently regular,
the characteristic ODE
\begin{equation}
\label{flow-ode}
\dot\gamma(t,x)=v(t,\gamma(t,x)),\qquad \gamma(0,x)=x,
\end{equation}
transports $\mu_0$ to $\mu_t$, i.e.,
$\mu_t=\gamma_{t,\sharp}\mu_0$
\citep[Proposition~8.1.8]{AmbrosioGigliSavare2005}.

\vspace{-3.5mm}\subsection{Euclidean flow matching}
\label{subsec:euclidean-flow-matching}

Flow matching turns the dynamic description above into mean-squared-error
(MSE) vector-field regression
\citep[see, e.g.,][]{lipman2023flowmatching_fm,wald2025flow}.
For $\mu,\nu\in\mP_2(\R^D)$ and
$\pi\in\mathrm{c}_{\R^D}(\mu,\nu)$, define
$\mu_t\coloneqq\pr^t_\sharp\pi$ with
$\pr^t(x,x')\coloneqq(1-t)x+tx'$. Then $(\mu_t,v_t^\pi)$ solves \eqref{ce1}, where $v^\pi$ minimizes
\begin{equation}
\label{minvr}
J_\pi(v)\coloneqq
\int_I\int_{\R^D\times\R^D}
\bigl\|v_t(\pr^t(x,x'))-(x'-x)\bigr\|^2
\,\d\pi(x,x')\,\d t
\end{equation}
over jointly Borel vector fields $v$.
Equivalently, for $(X,X')\sim\pi$, we have that
\begin{equation}
\label{eq:euclidean-fm-conditional-velocity}
v_t^\pi(z)
=
\mathbb{E}_\pi\bigl[X'-X\mid \pr^t(X,X')=z\bigr]
\end{equation}
for $\d t\,\d\mu_t(z)$-a.e.\ $(t,z)$. If
$\pi \in \mathrm{c}_{\R^D}^{\rm opt}(\mu,\nu)$, we have that
\(
\int_I\|v_t^\pi(z)\|^2_{L^2(\mu_t)}\,\d t\
{=}
\W_{2,\R^D}^2(\mu,\nu).
\)

\vspace{-3mm}\paragraph{Categorical endpoint prediction}
\label{subsec:endpoint-euclidean-flow-matching}

Let \(\mu,\nu\in\mP_2(\R^D)\) and
\(\pi\in\mathrm{c}_{\R^D}(\mu,\nu)\). Suppose that the target is
coordinatewise categorical. Write \([m]\coloneqq\{1,\ldots,m\}\) for a
positive integer \(m\), and fix a positive integer \(M\). For every
\(d\in[D]\), fix distinct values
\(a_{d,1},\ldots,a_{d,M}\in\R\) such that
\(
\nu\!\left(
\prod_{d=1}^D\{a_{d,1},\ldots,a_{d,M}\}
\right)=1.
\)
For \((X,X')\sim\pi\), set
\(Z_t\coloneqq(1-t)X+tX'\) and \(\mu_t\coloneqq\operatorname{Law}(Z_t)\).
Let \(\kappa_d(x')\in[M]\) denote the unique index such that
\(x'_d=a_{d,\kappa_d(x')}\). Choosing jointly Borel versions, define the
conditional coordinate probabilities
\[
q_t^{\ast,d}(n\mid z)
\coloneqq
\mathbb P_\pi\!\left(X'_d=a_{d,n}\mid Z_t=z\right),
\qquad d\in[D],\quad n\in[M].
\]
For a.e.\ \(t<1\) and \(\mu_t\)-a.e.\ \(z\),
\eqref{eq:euclidean-fm-conditional-velocity} yields
\begin{equation}
\label{eq:euclidean-fm-coordinate-endpoint-mixture}
\bigl(v_t^\pi(z)\bigr)_d
=
\frac{1}{1-t}
\left(
\sum_{n=1}^M a_{d,n}q_t^{\ast,d}(n\mid z)-z_d
\right).
\end{equation}

Moreover, \((q^{\ast,d})_{d=1}^D\) minimizes the coordinatewise
cross-entropy objective
\begin{equation}
\label{eq:euclidean-fm-coordinate-cross-entropy}
 J_\pi^{\rm coord}\bigl((q^d)_{d=1}^D\bigr)
\coloneqq
-\int_I\int_{\R^D\times\R^D}\sum_{d=1}^D
\log q_t^d\!\left(\kappa_d(x')\mid (1-t)x+tx'\right)
\,\d\pi(x,x')\,\d t
\end{equation}
over jointly Borel kernels \(q^d:I\times\R^D\to\Delta_M\),
\(d\in[D]\), where
\(\Delta_M\coloneqq\{p\in[0,1]^M:\sum_{n=1}^M p_n=1\}\) and
\(-\log 0\coloneqq+\infty\).
Thus, although the joint endpoint takes up to \(M^D\) values, it suffices
to train \(D\) distributions over \(M\) classes using
\eqref{eq:euclidean-fm-coordinate-cross-entropy} and recover the velocity
via \eqref{eq:euclidean-fm-coordinate-endpoint-mixture}
without any conditional independence assumption
\citep{chemseddine2026spherical,eijkelboom2024variational}.

\vspace{-3mm}\paragraph{Few-step distillation}
Optimal-transport couplings can improve few-step inference by simplifying the
teacher trajectories. Complementarily, recent flow-map methods learn finite-time
updates of a given flow directly
\citep{geng2026mean_meanflow,lee2026flow,roos2026categorical_flow_maps},
where the distilled teacher flow still depends on the coupling $\pi$ used for
its training. Write $q^{\theta_\pi}$ for the corresponding frozen categorical
teacher. For the categorical parametrization above, let
$\delta^{\phi,d}_{s,t}(\cdot\mid z)\in\Delta_M$ and
$D^\phi_{s,t}(z)_d\coloneqq\sum_{n=1}^M a_{d,n}\delta^{\phi,d}_{s,t}(n\mid z)$.
Following \citep{lee2026flow}, define
$\Phi^\phi_{s,t}(z)\coloneqq\frac{1-t}{1-s}z+\frac{t-s}{1-s}D^\phi_{s,t}(z)$
and train it to match the teacher on the diagonal and satisfy
$\Phi_{s,t}=\Phi_{u,t}\circ\Phi_{s,u}$. For $s<u<t$, set
$\lambda\coloneqq\frac{(1-t)(u-s)}{(1-u)(t-s)}$ and
$\bar\delta^{\phi,d}_{s,u,t}\coloneqq\operatorname{sg}\!\left[
\lambda\delta^{\phi,d}_{s,u}(\cdot\mid Z_s)
+(1-\lambda)\delta^{\phi,d}_{u,t}(\cdot\mid\Phi^\phi_{s,u}(Z_s))\right]$.
With $Z_r=(1-r)x+rx'$ and independently sampled $(x,x')\sim\mu\otimes\nu$, let
$\ell_{\rm comp}\coloneqq\sum_d\KL(\bar\delta^{\phi,d}_{s,u,t}\|
\delta^{\phi,d}_{s,t}(\cdot\mid Z_s))$ and
$\ell_{\rm diag}\coloneqq\sum_d\KL(q^{\theta_\pi,d}_r(\cdot\mid Z_r)\|
\delta^{\phi,d}_{r,r}(\cdot\mid Z_r))$. We minimize
\begin{equation}
\label{eq:flow_map_loss}
\int_{\mathcal T}\!\int \ell_{\rm comp}\,\d(\mu\otimes\nu)\,\d\rho
+\int_I\!\int \ell_{\rm diag}\,\d(\mu\otimes\nu)\,\d\rho_0 ,
\end{equation}
where $\rho$ is supported on
$\mathcal T\coloneqq\{(s,u,t):0\le s<u<t\le1\}$, $\rho_0$ on $I$, and
$\operatorname{sg}$ denotes the stop-gradient operation.
At zero loss, the diagonal term matches the teacher predictor and the first
term enforces flow-map composition. Inference iterates
$Z_{t_{k+1}}=\Phi^\phi_{t_k,t_{k+1}}(Z_{t_k})$.

\vspace{-3.5mm}\section{Flow matching on quotient spaces}
\label{sec:quotient-theory}

Let \(G\subset O(D)\) be compact, let \(Q\coloneqq\R^D/G\), and write
\(\mathfrak q(x)=[x]\) for the quotient map. The quotient metric is
\[
d_Q([x],[y])\coloneqq\min_{g\in G}\|x-gy\|,
\]
and \(\W_{2,Q}\) denotes the induced Wasserstein distance. For
\(\bmu\in\mP_2(Q)\), write
\[
[\bmu]\coloneqq
\{\mu\in\mP_2(\R^D):\mathfrak q_\sharp\mu=\bmu\}
\]
for its representative laws. In the graph setting, simultaneous node
relabeling gives the following construction. Let \(Z\coloneqq\R^C\) with
\(d_Z(u,v)\coloneqq\|u-v\|_2\). For \(E,F\in Z^{N\times N}\),
\begin{equation}
\label{eq:gromov_monge_as_euclidean_quoitient}
\GM_{2,Z}^2([E],[F])
\coloneqq
\min_{\sigma\in\mathfrak S_N}
\sum_{i,j=1}^N
d_Z^2\bigl(E_{ij},F_{\sigma(i)\sigma(j)}\bigr).
\end{equation}
Thus, choosing representatives of graph orbits is the (unnormalized) hard
Gromov--Monge node-alignment problem, studied in optimal transport \citep{memoli2024comparison}.

\begin{theorem}[Representative lifts and quotient geodesics]
\label{thm:quotient-lift-geodesic}
Let \(\bmu,\bnu\in\mP_2(Q)\), fix \(\mu\in[\bmu]\), and let
\(\Gamma\in\mathrm{c}_Q(\bmu,\bnu)\). There exist
\(\nu_\Gamma\in[\bnu]\) and
\(\gamma_\Gamma\in\mathrm{c}_{\R^D}(\mu,\nu_\Gamma)\) such that
\((\mathfrak q,\mathfrak q)_\sharp\gamma_\Gamma=\Gamma\) and
\[
\int_{\R^D\times\R^D}\|x-y\|^2\,\d\gamma_\Gamma(x,y)
=
\int_{Q\times Q}d_Q(\bx,\by)^2\,\d\Gamma(\bx,\by).
\]
If $\Gamma$ is optimal, then
\begin{equation}
\label{eq:quotient-wasserstein-lift}
\W_{2,Q}(\bmu,\bnu)
=
\W_{2,\R^D}(\mu,\nu_\Gamma)
=
\W_{2,\R^D}(\mu^G,\nu^G)
=
\min_{\tilde \nu\in[\bnu]}\W_{2,\R^D}(\mu,\tilde \nu),
\end{equation}
where $\nu\in[\bnu]$ is arbitrary and $\mu^G:=\mathbb E_{g\sim\lambda_G}[g_\sharp\mu]$ for $\lambda_G$ the normalized Haar measure on $G$.
Moreover, in this case, with
\(\pr^t(x,y)=(1-t)x+ty\) and
\(\mu_t\coloneqq\pr^t_\sharp\gamma_\Gamma\), the projected curve
\(\bmu_t\coloneqq\mathfrak q_\sharp\mu_t\) is a constant-speed
\(\W_{2,Q}\)-geodesic, and the flow-matching field associated with 
\(\gamma_\Gamma\) satisfies
\[
\int_I\int_{\R^D}\|v_t^{\gamma_\Gamma}(z)\|^2
\,\d\mu_t(z)\,\d t
=\W_{2,Q}^2(\bmu,\bnu).
\]
\end{theorem}

The proof is given in Appendix~\ref{app:measure-transport}.
The theorem shows that any quotient coupling can be lifted through aligned
representatives without increasing its cost. In particular, an optimal
quotient coupling admits a representative-space lift whose quadratic cost
equals the quotient Wasserstein distance. The resulting lift
\(\gamma_\Gamma\), however, need not be invariant under applying the same group
action to both endpoints, because selecting representatives can break the
symmetry. This matters because the velocity or endpoint network is constrained
to be \(G\)-equivariant. For any representative coupling \(\pi\)---in
particular, for \(\gamma_\Gamma\)---we therefore build a symmetrized coupling by
applying the same group element to both endpoints and averaging over \(G\). For
a law \(\eta\) and a coupling \(\pi\), set
\[
\eta^G\coloneqq\mathbb E_{g\sim\lambda_G}[g_\sharp\eta],
\qquad
\pi^G\coloneqq\mathbb E_{g\sim\lambda_G}[(g,g)_\sharp\pi],
\]
where the expectation denotes integration against the normalized Haar measure
on \(G\). For a finite group such as \(\mathfrak S_N\), this is the average over
all group elements. The resulting coupling is diagonally \(G\)-invariant,
meaning that \((h,h)_\sharp\pi^G=\pi^G\) for every \(h\in G\). It couples
\(\mu^G\) and \(\nu^G\) and represents the same quotient coupling as \(\pi\):
\((\mathfrak q,\mathfrak q)_\sharp\pi^G
=(\mathfrak q,\mathfrak q)_\sharp\pi\).
We next show that, on the \(G\)-equivariant model class, training with \(\pi\)
is equivalent to training with \(\pi^G\).

\begin{theorem}[Equivariant flow matching via symmetrization]
\label{thm:equivariant-symmetrization}
Let \(\pi\in\mathrm{c}_{\R^D}(\mu,\nu)\). Then:
\begin{enumerate}
\item For every \(G\)-equivariant field \(v\),
\(J_\pi(v)=J_{\pi^G}(v)\). The field
\(v^{\pi^G}\) can be chosen \(G\)-equivariant and minimizes
\(J_\pi\) over all \(G\)-equivariant fields.
\item If the ODE of a \(G\)-equivariant field has a unique flow
\(\gamma_t\), then it induces a well-defined quotient flow
\(
\bar\gamma_t([x])\coloneqq[\gamma_t(x)].
\)
For every \(\eta\in\mP_2(\R^D)\),
\(
\mathfrak q_\sharp\gamma_{t,\sharp}\eta
=\bar\gamma_{t,\sharp}\mathfrak q_\sharp\eta.
\)
Hence the projected curve depends only on the quotient law
\(\mathfrak q_\sharp\eta\), not on the representative lift chosen at time
\(0\).
\item If \(G\) consists of coordinate-permutation matrices and
\(\nu\) is supported on
\(\mathcal A^D\), \(\mathcal A=\{a_1,\ldots,a_M\}\), let
\[
q_t^{\ast,\pi^G,d}(n\mid z)
\coloneqq
\mathbb P_{\pi^G}(X'_d=a_n\mid Z_t=z).
\]
For \(g\in G\), let \(g\cdot d\) be the coordinate determined by
\((gz)_{g\cdot d}=z_d\). These coordinate posteriors can be chosen
equivariantly for \(d\in[D]\) and \(n\in[M]\):
\[
q_t^{\ast,\pi^G,g\cdot d}(n\mid gz)
=q_t^{\ast,\pi^G,d}(n\mid z).
\]
They minimize \( J_\pi^{\rm coord}\) over equivariant coordinate
kernels and recover \(v^{\pi^G}\) via
\eqref{eq:euclidean-fm-coordinate-endpoint-mixture}.
\end{enumerate}
\end{theorem}

The full proof is given in Appendix~\ref{app:equivariant-theory}.
Together, Theorems~\ref{thm:quotient-lift-geodesic}
and~\ref{thm:equivariant-symmetrization} connect the practical recipe used
below to quotient transport: align representatives and train an equivariant
Euclidean model whose flow, when well posed, descends to unordered graphs\textcolor{black}{.}

\vspace{-3.5mm}
\section{Inner and Outer Graph Alignment}
\label{sec:graph-couplings}
For \(C=C_{\rm e}+C_{\rm v}\), encode a graph by
\(E\in(\R^C)^{N\times N}\) with entries
$
E_{ik}~=~(
\sqrt{\lambda_{\rm edge}}\,e^E_{ik},
\mathbf 1_{\{i=k\}}
\sqrt{{\lambda_{\rm node}}/{C_{\rm v}}}\,f_i^E
),
$
where \(e^E_{ik}\in\R^{C_{\rm e}}\) contains edge features and
\(f_i^E\in\R^{C_{\rm v}}\) contains node features on the diagonal. Hence edge
and node-feature discrepancies enter the squared Euclidean cost with
weights \(\lambda_{\rm edge}\) and \(\lambda_{\rm node}\), respectively, as in
fused Gromov--Wasserstein transport \citep{vayer2020fused}. We use
\(\lambda_{\rm edge}=\lambda_{\rm node}=\tfrac12\) in all experiments. Because the same permutation acts on both node indices, the quotient cost
\eqref{eq:gromov_monge_as_euclidean_quoitient} compares all pairwise graph
relations after consistent node alignment. Its exact evaluation is an
NP-hard quadratic assignment, so we approximate this inner alignment either
by Gromov--Wasserstein relaxation \citep{bauer2025zgw} or by minimizing over
a finite random subset of node permutations, optionally followed by an outer
minibatch assignment.

\vspace{-3mm}\paragraph{Inner alignment via GW}
For \(E,F\in(\R^C)^{N\times N}\), we use a generalized version \citep{bauer2025zgw} of the discrete
Gromov--Wasserstein relaxation
\citep{memoli2011gromov,peyre2016gromov}, allowing for vector-valued edge features:
\begin{equation}
\label{eq:def_gromov_wasserstein}
\operatorname{GW}_{2,Z}^2([E],[F])
\coloneqq
\min_{T\in\mathcal U_N}
\sum_{i,j,k,l=1}^N
\|E_{ik}-F_{jl}\|^2T_{ij}T_{kl},
\end{equation}
where
\(\mathcal U_N=\{T\ge0:T\mathbf1=T^\top\mathbf1=N^{-1}\mathbf1\}\).
Restricting \(T\) to scaled permutation matrices recovers the Gromov--Monge
objective, so \(\operatorname{GW}_{2,Z}^2\le
\GM_{2,Z}^2/N^2\). We approximately solve the non-convex problem
by Frank--Wolfe iterations \citep{vayer2020fused}: each iteration solves an exact linearized transport
problem and updates the soft plan by an exact line search. From the final plan
\(T\), we recover the node permutation
\[
\widehat\sigma_{\operatorname{GW}}
\in\operatorname*{arg\,max}_{\sigma\in\mathfrak S_N}
\sum_{i=1}^N T_{i,\sigma(i)},
\]
 using the Hungarian algorithm \citep{kuhn1955hungarian}. 
Using the recovered
permutation, we define the aligned representative by
$
F^{\mathrm{aligned}}
\coloneqq \widehat\sigma_{\rm GW}^{-1}\cdot F,
F^{\mathrm{aligned}}_{ij}
=F_{\widehat\sigma_{\rm GW}(i)\widehat\sigma_{\rm GW}(j)}
$. Since GW is a
relaxation, this may not be the globally optimal
hard alignment.

\vspace{-3mm}\paragraph{Inner alignment via Best-of-Augmentation (BoA)}
As a simple alternative, we sample $\mathcal A$ (distinct) permutations
$\sigma_1,\ldots,\sigma_{\mathcal A}\in\mathfrak S_N$ for data augmentation and choose
\[
\textcolor{black}{\widehat\sigma_{\mathcal A}
\in\operatorname*{arg\,min}_{\sigma\in\{\sigma_1,\ldots,\sigma_{\mathcal A}\}}
\sum_{i,j=1}^N
d_Z^2\bigl(E_{ij},F_{\sigma(i)\sigma(j)}\bigr).}
\]
\textcolor{black}{For
$\Sigma_{\mathcal A}=(\sigma_1,\ldots,\sigma_{\mathcal A})$, define}
\textcolor{black}{$\widehat c_{\mathcal A}(E,F;\Sigma_{\mathcal A})
\coloneqq \sum_{i,j}d_Z^2(E_{ij},
F_{\widehat\sigma_{\mathcal A}(i)\widehat\sigma_{\mathcal A}(j)})$.}
\textcolor{black}{This} is an upper bound on
$\GM_{2,Z}^2([E],[F])$. For $\mathcal A=1$, this reduces to random augmentation.

\vspace{-3mm}\paragraph{Outer alignment}
For a graph pair $(E,F)$, let
$\widehat\sigma(E,F)\in\mathfrak S_N$ denote the permutation returned by the
selected inner alignment and define
$\textcolor{black}{\widehat c(E,F)
\coloneqq\|E-\widehat\sigma(E,F)^{-1}\cdot F\|^2.}$
For independently sampled batches $(E^a)_{a=1}^B$ and $(F^b)_{b=1}^B$, we
either set $\widehat\tau=\mathrm{id}$ or seek to solve
\[
\widehat\tau\in\argmin_{\tau\in\mathfrak S_B}
\sum_{a=1}^B\textcolor{black}{\widehat c(E^a,F^{\tau(a)})}.
\]
For GW, we use the relaxed GW objective in place of exact $\widehat c$ when
computing the outer assignment, avoiding hard projection for all $B^2$ pairs.
For each selected pair, let $\widehat\sigma_a~\coloneqq~\widehat\sigma(E^a,F^{\widehat\tau(a)})$ and define
$\widehat\pi_B\coloneqq B^{-1}\sum_{a=1}^B
\delta_{(E^a,\widehat\sigma_a^{-1}\cdot F^{\widehat\tau(a)})}$.
Thus, $\widehat\sigma_a$ aligns nodes within each pair, while $\widehat\tau$
selects which graphs are paired.

\begin{algorithm}[H]
\caption{Inner and outer minibatch graph alignment}
\label{alg:minibatch_alignment}
\small
\begin{algorithmic}[1]
\REQUIRE Source batch $(E^a)_{a=1}^B$, target batch $(F^b)_{b=1}^B$,
inner alignment method
\IF{the outer coupling is independent}
    \STATE $\widehat\tau(a)\gets a$ for all $a=1,\ldots,B$
\ELSE
    \STATE $C_{ab}\gets\textcolor{black}{\widehat c(E^a,F^b)}$ for all $a,b$
    \STATE $\widehat\tau\gets
    \argmin_{\tau\in\mathfrak S_B}\sum_a C_{a,\tau(a)}$
\ENDIF
\FOR{$a=1,\ldots,B$}
    \STATE $G^a\gets F^{\widehat\tau(a)}$
    \STATE Obtain $\widehat\sigma_a$ using the selected inner alignment
    \STATE $\widetilde F^a\gets\widehat\sigma_a^{-1}\cdot G^a$
\ENDFOR
\STATE \textbf{return} $(E^a,\widetilde F^a)_{a=1}^B$
\end{algorithmic}
\end{algorithm}

\vspace{-3mm}\paragraph{Training}
Given the aligned pairs, training proceeds as in standard flow matching.
Write \(\widetilde F^a=\widehat\sigma^{-1}_a\cdot
F^{\widehat\tau(a)}\), sample \(t^a\sim\mathcal U[0,1]\), and set
\(E_{t^a}^a=(1-t^a)E^a+t^a\widetilde F^a\). The continuous and categorical
minibatch losses are, respectively,
\begin{equation}
\label{eq:minibatch-fm-loss}
J_B(v^{\theta})
\coloneqq
\frac1B\sum_{a=1}^B
\left\|
v^{\theta}(t^a,E_{t^a}^a)-(\widetilde F^a-E^a)
\right\|^2.
\end{equation}
\begin{equation}
\label{eq:minibatch-endpoint-loss}
 J_B^{\rm coord}\bigl((q_\theta^d)_{d=1}^D\bigr)
\coloneqq
-\frac1B\sum_{a=1}^B
\sum_{d=1}^D
\log q_{\theta,t^a}^d\!\left(
\kappa_d(\widetilde F^a)\mid E_{t^a}^a
\right).
\end{equation}
For one-hot graph features, the categorical loss and velocity are applied
blockwise, see Remark \ref{rem:categorical_feature_blocks}.
For datasets containing
graphs with different numbers of nodes, one graph transformer is shared across
all node counts: it maps an \(N\)-node input to an \(N\)-node prediction, while
\(N\) may vary between batches. We form each training batch from graphs with
the same \(N\). At inference, we sample \(N\) from the empirical distribution
of node counts in the training set. Our graph transformers are permutation-equivariant. Implementation details:
Appendix~\ref{app:repro}.
\begin{remark}[\textbf{Expected Limit Behavior}]
\textit{
Let us first consider \textbf{inner alignment} only, where $[E]$ and $[F]$
are matched randomly and the representative
$\widehat\sigma(E,F)^{-1}\!\cdot F$ is aligned via BoA or GW.
For BoA with $\mathcal A$ sampled permutations, denote the resulting alignment
by $\widehat\sigma_{\mathcal A}$ and its \textcolor{black}{realized} cost by
\textcolor{black}{$\widehat c_{\mathcal A}(E,F;\Sigma_{\mathcal A})$}. Random alignment corresponds to
$\mathcal A=1$. Assuming a unique optimal GM alignment
$\sigma_{\operatorname{GM}}$, we expect
$\widehat\sigma_{\mathcal A}\to\sigma_{\operatorname{GM}}$ and
\textcolor{black}{$\widehat c_{\mathcal A}(E,F;\Sigma_{\mathcal A})
\to\operatorname{GM}_{2,Z}^2([E],[F])$} as $\mathcal A$ increases, with
equality when all $N!$ permutations are considered. \textcolor{black}{Our GW approximation 
need not attain this optimum, but we expect better
alignment than BoA for moderate $\mathcal A$.}
\\
For joint \textbf{inner and outer alignment}, let
$\widehat\tau_{\mathcal A}\in\mathfrak S_B$ denote the outer graph assignment
obtained from the BoA costs $\widehat c_{\mathcal A}$ and let
$\widehat\tau_{\operatorname{GW}}$ denote its GW analogue. \textcolor{black}{With exhaustive
inner search, the outer assignment minimizes the empirical GM cost. The
GW-based assignment need not coincide. Moreover, with exact inner alignment, i.i.d.\ sampling, and finite second
moments, the optimal average minibatch cost converges to
$\W_{2,Q}^2(\bmu,\bnu)$ as $B\to\infty$. By
\eqref{eq:quotient-wasserstein-lift}, Euclidean minibatch optimal transport {\color{black}($\textsc{MBOT}$)}
\citep{tong2023improving_minibatch} has 
the same limit 
when
$\mu=\mu^G$ and targets are uniformly relabeled upon each sampling,
thereby sampling from $\nu^G$.}
}
\end{remark}
\begin{figure}[b]
    \centering
    \setlength{\fboxsep}{0pt}
    \newcommand{\circpanel}[1]{\fbox{\adjincludegraphics[trim={0 {.05\height} 0 {.10\height}},clip,width=\dimexpr\linewidth-2\fboxrule\relax]{#1}}}
    \begin{minipage}[c]{0.19\textwidth}
        \centering
        \begin{subfigure}[t]{\linewidth}
            \centering\circpanel{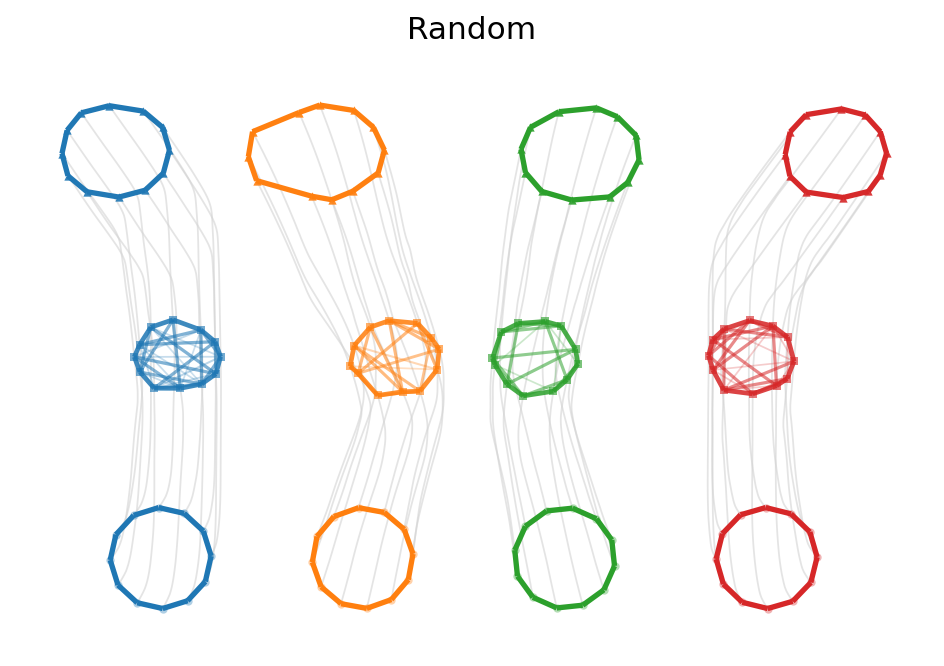}
            \caption*{\scriptsize\textsc{Random}}
        \end{subfigure}
    \end{minipage}\hfill
    \begin{minipage}[c]{0.785\textwidth}
        \centering
        \begin{subfigure}[t]{0.24\linewidth}
            \centering\circpanel{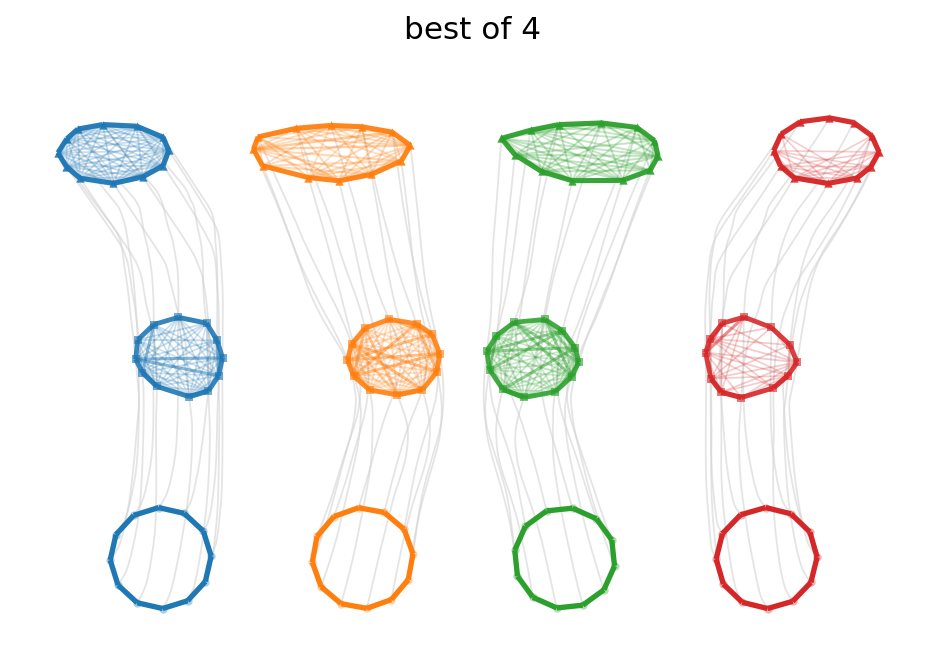}
            \caption*{\scriptsize\textsc{BoA}$^{4}$}
        \end{subfigure}\hfill
        \begin{subfigure}[t]{0.24\linewidth}
            \centering\circpanel{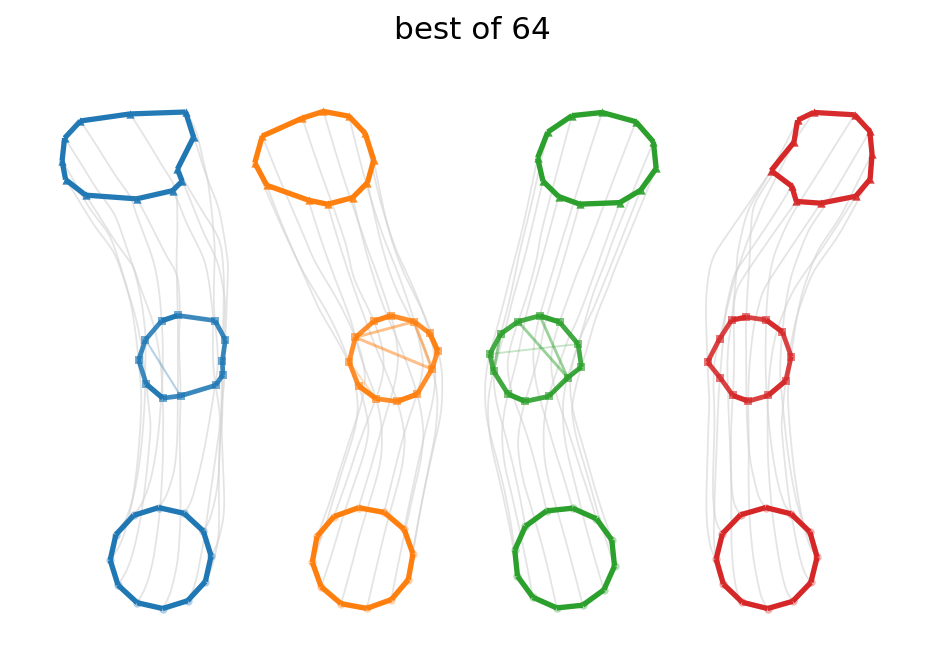}
            \caption*{\scriptsize\textsc{BoA}$^{64}$}
        \end{subfigure}\hfill
        \begin{subfigure}[t]{0.24\linewidth}
            \centering\circpanel{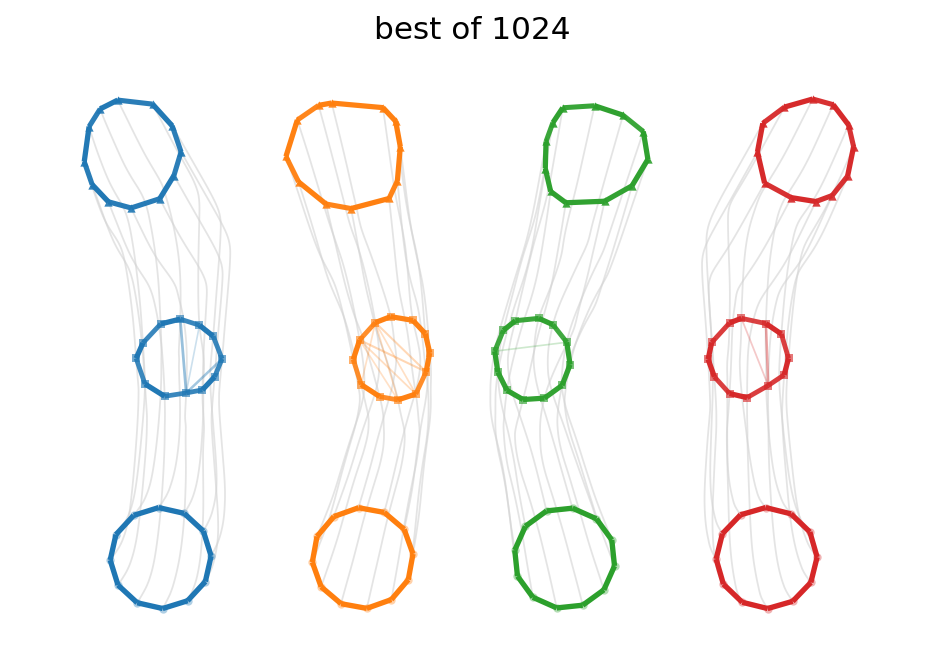}
            \caption*{\scriptsize\textsc{BoA}$^{1024}$}
        \end{subfigure}\hfill
        \begin{subfigure}[t]{0.24\linewidth}
            \centering\circpanel{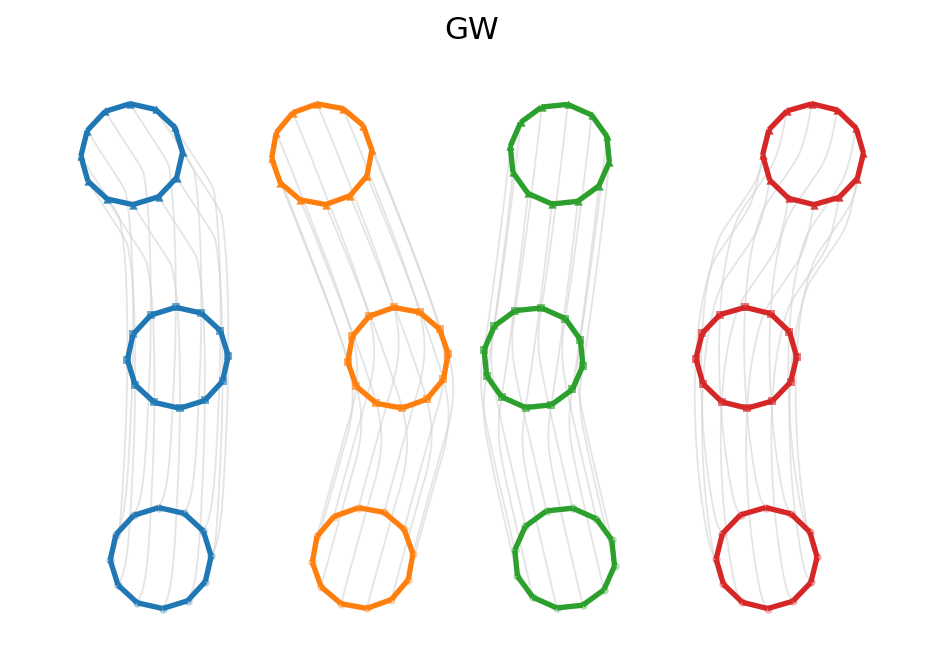}
            \caption*{\scriptsize\textsc{GW}}
        \end{subfigure}
        \\[4pt]
        \begin{subfigure}[t]{0.24\linewidth}
            \centering\circpanel{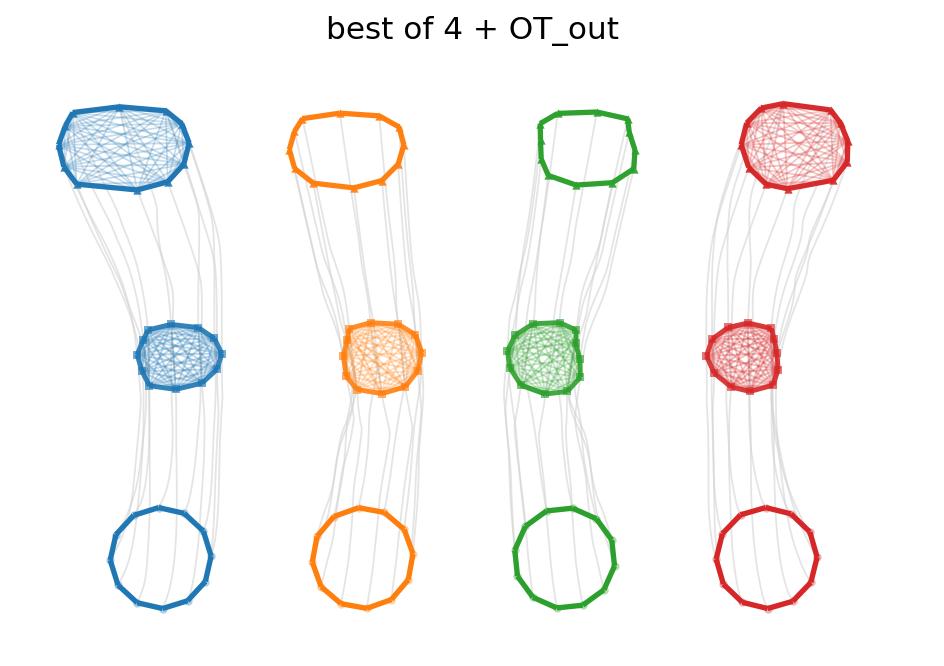}
            \caption*{\scriptsize\textsc{BoA}$^{4}$+\textsc{BoA}$^{\scriptscriptstyle 4}_{\mathrm{out}}$}
        \end{subfigure}\hfill
        \begin{subfigure}[t]{0.24\linewidth}
            \centering\circpanel{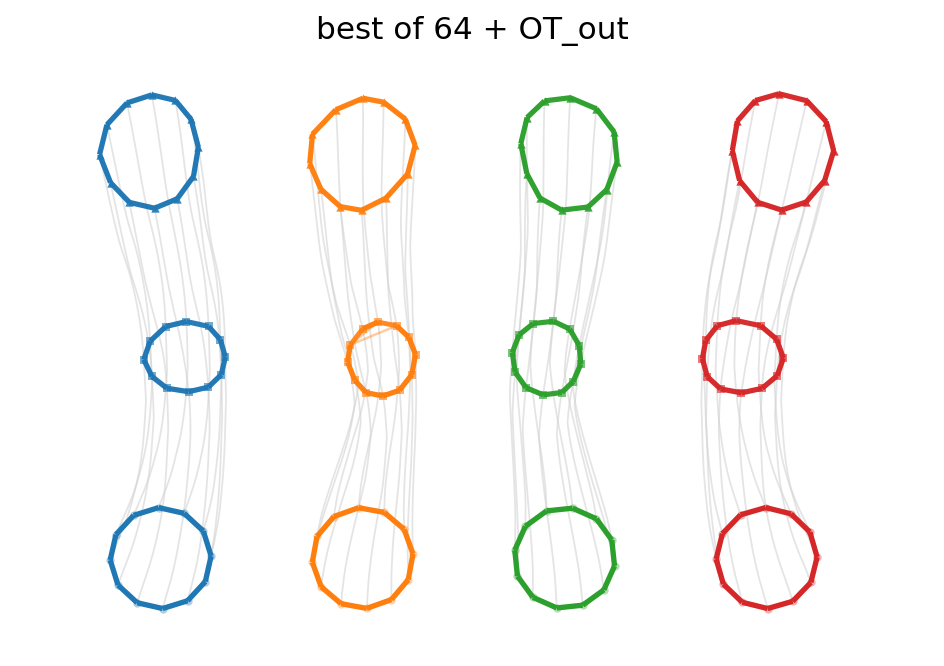}
            \caption*{\scriptsize\textsc{BoA}$^{64}$+\textsc{BoA}$^{\scriptscriptstyle 64}_{\mathrm{out}}$}
        \end{subfigure}\hfill
        \begin{subfigure}[t]{0.24\linewidth}
            \centering\circpanel{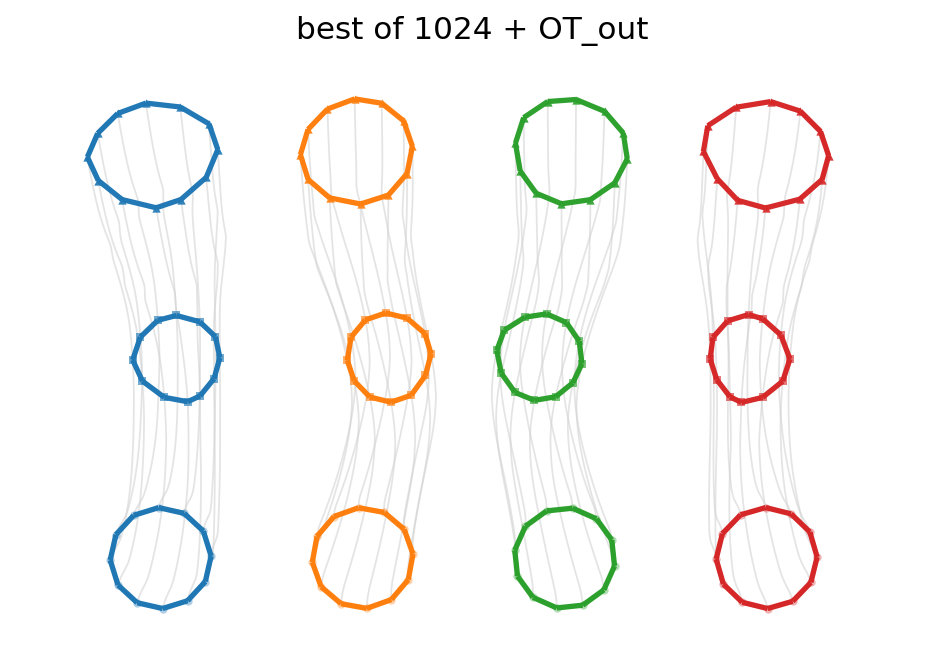}
            \caption*{\scriptsize\textsc{BoA}$^{1024}$+\textsc{BoA}$^{\scriptscriptstyle 1024}_{\mathrm{out}}$}
        \end{subfigure}\hfill
        \begin{subfigure}[t]{0.24\linewidth}
            \centering\circpanel{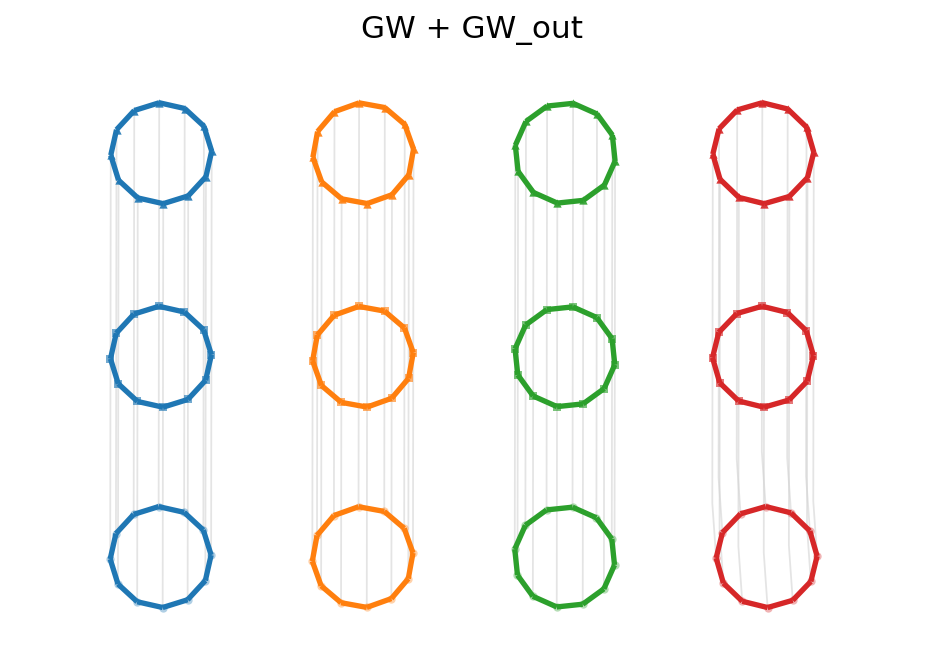}
            \caption*{\scriptsize\textsc{GW}+\textsc{GW}$^{16}_{\mathrm{out}}$}
        \end{subfigure}
    \end{minipage}
    \caption{Cycle-graph trajectories {\textcolor{black}{ \tiny ($t=0$: bottom, $t=0.5$: center, $t=1$: top)}}. Colours mark source, lines node paths, and edge width edge value. At $t=0.5$, dense graphs become sparse circles as augmentations $\mathcal A$ grow towards \textsc{GW} (top). Outer coupling straightens paths (bottom).}
    \label{fig:circle_graph_trajs}
\end{figure}
\begin{remark}[\textbf{Compute Scaling}]
\textit{
For fixed compute, inner alignment quality trades off with outer batch size.
Joint inner and outer BoA evaluates $\mathcal A$ permutations for each of the
$B^2$ graph pairs. Since one comparison costs
$O(C_{\rm e}N^2+C_{\rm v}N)$, as node features ($C_v$) occur only on the diagonal,
and the outer assignment costs $O(B^3)$, its total complexity is
$O\!\left(B^2\mathcal A(C_{\rm e}N^2+C_{\rm v}N)+B^3\right)$.
Thus, $\mathcal A$ and $B$ trade off inner and outer
alignment quality. Since the cost is linear in $\mathcal A$ but cubic in $B$,
additional inner alignment becomes relatively cheaper as $B$ grows.
In comparison, exposing each source to $\mathcal AB$ candidate target
representatives via \textsc{MBOT} requires
$B_{\textsc{MBOT}}\simeq\mathcal AB$, with complexity
$O(B_{\textsc{MBOT}}^2(C_{\rm e}N^2+C_{\rm v}N)+B_{\textsc{MBOT}}^3)$.
Hence increasing inner search $\mathcal A$ is computationally cheaper than
obtaining the same candidate-set growth solely through the outer batch size $B$.
A classical GW solver with $K$ iterations costs $O(KCN^3)$ per
graph pair, making joint alignment practical only for small $B$.
Here, a natural quantity is the mean aligned graph distance
$
\mathbb E_{E \sim \mu^G, F \sim \nu^G}
\|E^a-\widehat\sigma_a^{-1}\!\cdot F^{\widehat\tau(a)}\|^2$,
where the expectation includes the minibatch-induced randomness of
$\widehat{\tau}$ and $\widehat{\sigma}_a$, i.e., the transport cost.
Note that Sinkhorn solvers accelerate MBOT, but 
 return soft plans rather than permutations
 and entropic bias increases transport cost.
}
\end{remark}
\vspace{-3.5mm}\section{Numerical experiments}
\label{sec:experiments}
We evaluate the coupling constructions in a permutation-equivariant graph
flow-matching model. For inner alignment, we compare random node relabeling
(\textsc{Random}), BoA with $\mathcal A$ permutations
(\textsc{BoA}$^{\mathcal A}$), and a $10$-iteration Frank--Wolfe GW solver
(\textsc{GW}). As a permutation-blind reference, \textsc{MBOT}$^B$ randomly
relabels each target and performs Euclidean minibatch OT over $B$ graphs using
the cost $\|E-F\|^2$, without inner alignment. For joint inner--outer
alignment, $+\textsc{BoA}^{B}_{\rm out}$ assigns $B$ graphs using the
\textsc{BoA}$^{\mathcal A}$ costs 
($\textsc{MBOT}^{B}\simeq \textsc{BoA}^{1}+\textsc{BoA}^{B}_{\rm out}$), while
$+\textsc{GW}^{B}_{\rm out}$ uses the GW costs.
The alignment minibatch size $B$ is independent of the network gradient batch
size. Aligned minibatches are split or stacked such that every gradient step
uses the same number of training pairs. 
For $+\textsc{GW}^{B}_{\rm out}$, we
use $B\in\{8,16\}$ due to its computational cost. Alignment minibatches contain
graphs with the same number of nodes. Within each experiment, models share an equivariant graph-transformer
backbone and differ only in their training coupling. We
evaluate several Euler budgets to assess the effect of alignment on few-step
generation. Complete implementation details are given in
Appendix~\ref{app:repro}.
\\
For variable-size datasets, the population construction first draws the node
count $N\sim p_N$ from its empirical distribution. Conditional on $N$,
\textsc{Random} uses the product coupling $\mu_N\otimes\nu_N$, where
$\mu_N$ and $\nu_N$ denote the source and target laws for $N$ nodes.
Inner \textsc{BoA}/\textsc{GW} changes only the representation of  target graphs, whereas outer variants also optimize the
source--target pairing within fixed-$N$ minibatches. See Appendix~\ref{sbm:appendix_sec} for further experiments.
\vspace{-3.5mm}\subsection{Continuous target via translated cycle graphs}
\label{subsec:circles}
Inspired by \citep{piening2026generalized_wow_fm}, we consider randomly relabelled \(12\)-node rings represented by adjacency and planar 2D node positions. Source and target differ only by translation and independent rotation, so they have the same edge law and any intermediate deformation is induced solely by the coupling. To give the source a density and stabilize training, graph sources are perturbed by small Gaussian noise (\(\sigma=0.01\)) after coupling. At test time, we start from exact rings and integrate with \(125\) Euler steps. We report roundness, the standard deviation of node distances to the centroid (rescaled by \(\times 10^{2}\), \(0\) for an exact ring), and the mean degree \(\frac1N\sum_{i\neq j}\mathrm{Adj}_{ij}\) of the weighted adjacency matrix, which should be \(2\).
\\
Figure~\ref{fig:circle_graph_trajs} and Table~\ref{tab:circle_summary} show the results for $125$ Euler steps. Under \textsc{Random}, mismatched edges blur rings halfway along the flow, and generated rings are strongly deformed. \textcolor{black}{For \textsc{BoA}, roundness improves steadily as \(\mathcal A\) grows, though small \(\mathcal A\) can underperform \textsc{Random}. \textsc{GW} reaches near-exact rings.} Adding the outer assignment straightens the paths and improves roundness at every \(\mathcal A\), and \textsc{GW}\(+\)\textsc{GW}\(_{\mathrm{out}}\) recovers exact rings. The permutation-blind \textsc{MBOT} also improves with \(B\), but stays far from \textsc{GW}. \textcolor{black}{The two GW-based alignments give the best results and recover exact circles.
Inner--outer GW induces continuous translation of circles.
}

\begin{table}[t]
\vspace{-3mm}
\centering
\footnotesize
\setlength{\tabcolsep}{3pt}
\caption{Cycle-graphs: Roundness and mean degree are reported as mean$_{\mathrm{std}}$. Best is bold.}
\vspace{-3mm}
\label{tab:circle_summary}
\resizebox{\linewidth}{!}{%
\begin{tabular}{l c|ccc|cccc|cccc}
\toprule
Method & {\scriptsize \textsc{Random}} & {\scriptsize \textsc{MBOT}$^{4}$} & {\scriptsize \textsc{MBOT}$^{64}$} & {\scriptsize \textsc{MBOT}$^{1024}$} & {\scriptsize \textsc{BoA}$^{4}$} & {\scriptsize \textsc{BoA}$^{64}$} & {\scriptsize \textsc{BoA}$^{1024}$} & {\scriptsize \textsc{GW}} & {\scriptsize \textsc{BoA}$^{4}$} & {\scriptsize \textsc{BoA}$^{64}$} & {\scriptsize \textsc{BoA}$^{1024}$} & {\scriptsize \textsc{GW}}\\
 &  &  &  &  &  &  &  &  & {\scriptsize $+$\textsc{BoA}$^{4}_{\mathrm{out}}$} & {\scriptsize $+$\textsc{BoA}$^{64}_{\mathrm{out}}$} & {\scriptsize $+$\textsc{BoA}$^{1024}_{\mathrm{out}}$} & {\scriptsize $+$\textsc{GW}$^{16}_{\mathrm{out}}$}\\
\midrule
Roundness ($\downarrow$) & $12.9_{8.9}$ & $18.9_{4.9}$ & $8.8_{4.0}$ & $7.1_{3.0}$ & $26.1_{7.0}$ & $16.4_{7.2}$ & $7.9_{3.2}$ & $0.7_{1.0}$ & $11.4_{3.0}$ & $5.9_{1.2}$ & $5.5_{2.0}$ & $\mathbf{0.0}_{0.0}$\\
Degree ($\to 2$) & $1.90_{0.28}$ & $2.15_{0.09}$ & $1.81_{0.14}$ & $1.80_{0.12}$ & $1.91_{0.09}$ & $1.93_{0.13}$ & $1.94_{0.03}$ & $\mathbf{2.00}_{0.00}$ & $1.96_{0.08}$ & $1.87_{0.02}$ & $1.97_{0.01}$ & $\mathbf{2.00}_{0.00}$\\
\bottomrule
\end{tabular}
}
\vspace{-4mm}
\end{table}

\vspace{-3.5mm}\subsection{Categorical target via molecular graph generation}
\label{subsec:molecular}
\vspace{-1mm}\paragraph{Data}
We use QM9 (\(N\le9\)) \citep{ramakrishnan2014quantum_qm9} and ZINC250k (\(N\le38\)) \citep{jin2018junction_zinc} in the categorical
CatFlow/DiGress representation \citep{eijkelboom2024variational,
vignac2023digress}. Edges are one-hot over
\(\{\text{none},\text{single},\text{double},\text{triple}\}\), while each
diagonal node feature concatenates a one-hot atom type with a formal-charge
one-hot over \(\{-1,0,+1\}\). QM9 uses atom types
\(\{\mathrm C,\mathrm N,\mathrm O,\mathrm F\}\), giving \(C=11\). ZINC250k
uses \(\{\mathrm C,\mathrm N,\mathrm O,\mathrm F,\mathrm{Br},\mathrm{Cl},
\mathrm I,\mathrm P,\mathrm S\}\), giving \(C=16\). The network predicts each categorical feature block and minimises
the blockwise loss \eqref{eq:minibatch-endpoint-loss}. Its velocity
uses the blockwise conditional-mean 
in Remark \ref{rem:categorical_feature_blocks}.
\vspace{-3mm}\paragraph{Metrics}
Following the GDSS/CatFlow evaluation protocols
\citep{eijkelboom2024variational,jo2022gdss}, we report
uncorrected molecular validity, the fraction of generated graphs that RDKit
sanitizes without correction, and Fr\'echet
ChemNet Distance (FCD). Validity 
captures chemical correctness and FCD distributional fidelity. Higher is better for validity and lower is better for
FCD. 
We follow \citep{eijkelboom2024variational} and omit novelty since it might reward unfaithful modeling for the selected benchmark datasets. Similarly, we defer uniqueness reporting to Appendix~\ref{sec:mol-sweep_uniqueness}, as it is saturated and uninformative.
\vspace{-3mm}\paragraph{Compute scaling of the alignment}
For fixed $N$, we sample $32{,}768$ source graphs with relabeled QM9/ZINC250k
targets and record $\mathbb E_{E \sim \mu^G, F \sim \nu^G}
\|E^a-\widehat\sigma_a^{-1}\!\cdot F^{\widehat\tau(a)}\|^2$ against CPU time (Fig.~\ref{fig:compute_pareto}), where we split the data into $B$-sized chunks. The
distance depends almost only on the budget $\mathcal AB$, but joint splits
reach it fastest. On QM9 ($N=9$),
\textsc{BoA}$^{64}+\textsc{BoA}^{512}_{\rm out}$ matches
\textsc{MBOT}$^{32768}$ {\color{black}but is} $12\times$ faster. The tradeoff  between $\mathcal A$ and $B$ depends on $N$ and
the feature dimension $C$. On QM9 splits ($N\in\{8,9\}$, $C=11$), \textsc{BoA}$^{\mathcal A}$
approaches exhaustive search over $\mathfrak S_N$ and lies below \textsc{GW}.
On ZINC250k splits ($N\in\{18,23\}$, $C=16$), random search in $\mathfrak S_N$ seems hopeless
and \textcolor{black}{\textsc{GW} is clearly best.}
\begin{figure}[h]
    \vspace{-4mm}
    \centering
    \includegraphics[width=\linewidth]{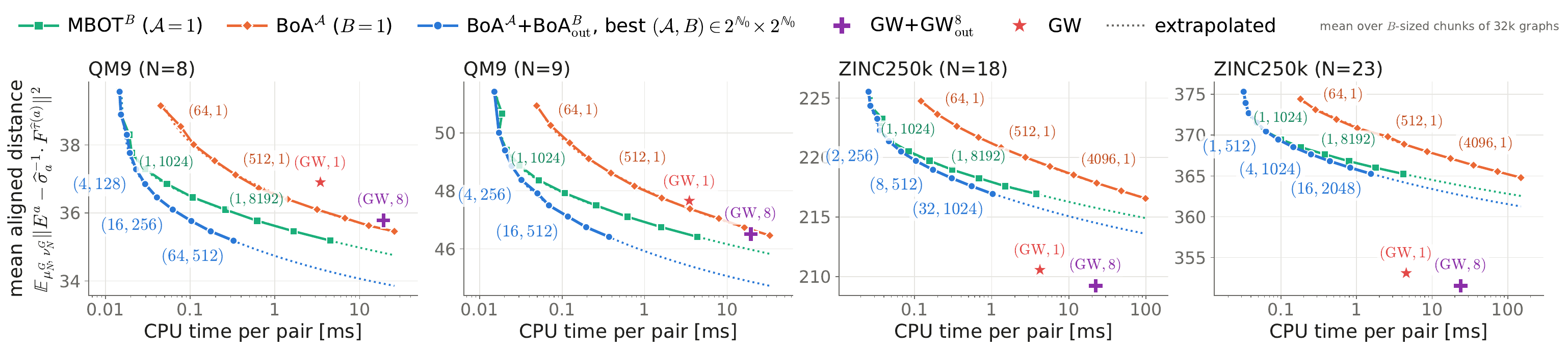}
    \vspace{-5mm}
    \caption{Mean aligned distance vs.\ CPU time per pair. Labels give
    $(\mathcal A,B)$. Dotted: extrapolated least-squares $\mathcal AB$ power-law fits
    (distance: $c_\infty+a(\mathcal AB)^{-p}$, time: $t_0+b(\mathcal AB)^{q}$).}
    \label{fig:compute_pareto}
\end{figure}
\vspace{-7mm}\paragraph{Limited-budget sweep of endpoint-induced velocity field and flow map}
We train the shared 2.8M-parameter backbone for 100 epochs under each coupling
and evaluate 10,000 samples at \(5/25/125\) Euler steps. We additionally distill
each model into a flow map using the same backbone and
\eqref{eq:flow_map_loss}, evaluated at \(1/2/5\) steps, see
Appendix~\ref{app:repro}.
At low ODE budgets (Table~\ref{tab:mol-sweep_ODE}), GW alignment consistently
improves validity and FCD over random coupling, with the largest gains at five
steps. Outer assignment further improves validity, while its FCD benefit is
dataset dependent. \textsc{MBOT} is competitive on QM9 but weaker than GW in
FCD and on ZINC250k. Differences narrow as the integration budget increases.
For the distilled flow map (Table~\ref{tab:mol-sweep_flow_map}), coupling
differences are smaller, but GW-based alignment displays an advantage for two or five steps.
\begin{table}[h]
\vspace{-2mm}
  \centering
  \scriptsize
  \setlength{\tabcolsep}{1pt}
  \caption{Molecular short-budget sweep: (\ref{tab:mol-sweep_ODE}) Flow ODE at
  $5/25/125$ Euler steps and (\ref{tab:mol-sweep_flow_map}) flow map at $1/2/5$
  steps. Rows are mean$_{\text{std}}$ over $3$ evaluation repeats. Best {\tiny ($\pm2\sigma$)} is
  bold.}
  \vspace{-3mm}
  \label{tab:mol-sweep}
  \begin{subtable}[t]{0.49\textwidth}
  \centering
  \caption{\scriptsize Euler Integration of Learned Flow ODE}
  \vspace{-1.5mm}
  \label{tab:mol-sweep_ODE}
  \begin{adjustbox}{max width=\linewidth}
  \begin{tabular}{l ccc ccc}
  \toprule
  & \multicolumn{3}{c}{Validity ($\uparrow$)}
  & \multicolumn{3}{c}{FCD ($\downarrow$)}\\
  \cmidrule(lr){2-4}\cmidrule(lr){5-7}
  Method & $5$ & $25$ & $125$ & $5$ & $25$ & $125$\\
  \midrule
  \multicolumn{7}{l}{\emph{QM9} ($N\le9$)}\\
  \textsc{Random}
   & $0.8849_{.0020}$ & $0.9564_{.0002}$ & $0.9648_{.0010}$
   & $1.731_{.017}$ & $0.637_{.022}$ & $0.530_{.012}$
   \\
  \textsc{MBOT}$^{1024}$
   & $0.9327_{.0016}$ & $\mathbf{0.9716}_{.0018}$ & $\mathbf{0.9785}_{.0012}$
   & $1.581_{.018}$ & $\mathbf{0.585}_{.051}$ & $0.494_{.019}$\\
  \textsc{BoA}$^{4}$$+$\textsc{BoA}$^{256}_{\mathrm{out}}$
   & $0.9283_{.0041}$ & $\mathbf{0.9735}_{.0013}$ & $\mathbf{0.9769}_{.0005}$
   & $1.621_{.031}$ & $0.631_{.003}$ & $\mathbf{0.511}_{.034}$\\
  \textsc{GW}
   & $0.9356_{.0020}$ & $0.9695_{.0010}$ & $0.9744_{.0014}$
   & $\mathbf{1.278}_{.046}$ & $0.589_{.015}$ & $\mathbf{0.491}_{.022}$\\
  \textsc{GW}$+$\textsc{GW}$^8_{\mathrm{out}}$
   & $\mathbf{0.9421}_{.0024}$ & $\mathbf{0.9728}_{.0011}$ & $\mathbf{0.9783}_{.0009}$
   & $\mathbf{1.232}_{.017}$ & $\mathbf{0.532}_{.014}$ & $\mathbf{0.440}_{.016}$\\
  \midrule
  \multicolumn{7}{l}{\emph{ZINC250k} ($N\le38$)}\\
  \textsc{Random}
   & $0.5641_{.0048}$ & $0.8140_{.0034}$ & $0.8597_{.0011}$
   & $18.448_{.095}$ & $13.061_{.162}$ & $11.909_{.077}$\\
  \textsc{MBOT}$^{1024}$
   & $0.6405_{.0034}$ & $0.8343_{.0013}$ & $0.8671_{.0003}$
   & $17.383_{.009}$ & $12.628_{.127}$ & $11.685_{.033}$\\
  \textsc{BoA}$^{2}$$+$\textsc{BoA}$^{512}_{\mathrm{out}}$
   & $0.6472_{.0035}$ & $0.8541_{.0009}$ & $0.8815_{.0004}$
   & $17.366_{.115}$ & $12.470_{.165}$ & $11.612_{.071}$\\
  \textsc{GW}
   & $0.6431_{.0029}$ & $0.8472_{.0052}$ & $0.8742_{.0009}$
   & $\mathbf{15.024}_{.076}$ & $\mathbf{11.993}_{.030}$ & $\mathbf{11.365}_{.073}$\\
  \textsc{GW}$+$\textsc{GW}$^8_{\mathrm{out}}$
   & $\mathbf{0.7280}_{.0055}$ & $\mathbf{0.8763}_{.0026}$ & $\mathbf{0.8930}_{.0022}$
   & $15.377_{.048}$ & $12.116_{.025}$ & $\mathbf{11.370}_{.059}$\\
  \bottomrule
  \end{tabular}
  \end{adjustbox}
  \end{subtable}
  \hfill
  \begin{subtable}[t]{0.49\textwidth}
  \centering
  \caption{\scriptsize Distilled Flow Map for Few-Step Generation}
  \vspace{-1.5mm}
  \label{tab:mol-sweep_flow_map}
  \begin{adjustbox}{max width=\linewidth}
  \begin{tabular}{l ccc ccc}
  \toprule
  & \multicolumn{3}{c}{Validity ($\uparrow$)}
  & \multicolumn{3}{c}{FCD ($\downarrow$)}\\
  \cmidrule(lr){2-4}\cmidrule(lr){5-7}
  Method & $1$ & $2$ & $5$ & $1$ & $2$ & $5$\\
  \midrule
  \multicolumn{7}{l}{\emph{QM9} ($N\le9$)}\\
  \textsc{Random}
   & $0.8183_{.0049}$ & $0.8715_{.0023}$ & $0.9040_{.0029}$
   & $\mathbf{1.576}_{.011}$ & $1.295_{.022}$ & $0.839_{.021}$\\
  \textsc{MBOT}$^{1024}$
   & $\mathbf{0.8499}_{.0030}$ & $0.8955_{.0004}$ & $\mathbf{0.9447}_{.0006}$
   & $\mathbf{1.546}_{.022}$ & $1.299_{.042}$ & $0.835_{.016}$\\
  \textsc{BoA}$^{4}$$+$\textsc{BoA}$^{256}_{\mathrm{out}}$
   & $0.8353_{.0023}$ & $0.8934_{.0010}$ & $\mathbf{0.9424}_{.0024}$
   & $\mathbf{1.579}_{.013}$ & $1.224_{.018}$ & $0.866_{.038}$\\
  \textsc{GW}
   & $0.8381_{.0051}$ & $\mathbf{0.9056}_{.0029}$ & $\mathbf{0.9447}_{.0020}$
   & $\mathbf{1.555}_{.042}$ & $\mathbf{1.187}_{.046}$ & $0.827_{.013}$\\
  \textsc{GW}$+$\textsc{GW}$^8_{\mathrm{out}}$
   & $\mathbf{0.8496}_{.0041}$ & $\mathbf{0.9085}_{.0028}$ & $\mathbf{0.9474}_{.0020}$
   & $\mathbf{1.529}_{.049}$ & $\mathbf{1.175}_{.015}$ & $\mathbf{0.766}_{.016}$\\
  \midrule
  \multicolumn{7}{l}{\emph{ZINC250k} ($N\le38$)}\\
  \textsc{Random}
   & $0.4608_{.0008}$ & $0.6016_{.0032}$ & $0.6611_{.0024}$
   & $18.561_{.133}$ & $16.595_{.064}$ & $15.496_{.063}$\\
  \textsc{MBOT}$^{1024}$
   & $0.4646_{.0056}$ & $0.6308_{.0056}$ & $0.6865_{.0022}$
   & $\mathbf{18.110}_{.104}$ & $\mathbf{16.137}_{.076}$ & $\mathbf{14.511}_{.038}$\\
  \textsc{BoA}$^{2}$$+$\textsc{BoA}$^{512}_{\mathrm{out}}$
   & $\mathbf{0.4972}_{.0029}$ & $\mathbf{0.6758}_{.0001}$ & $0.7354_{.0023}$
   & $\mathbf{18.305}_{.086}$ & $16.322_{.047}$ & $14.701_{.100}$\\
  \textsc{GW}
   & $0.4418_{.0039}$ & $0.6238_{.0056}$ & $0.7302_{.0045}$
   & $\mathbf{18.198}_{.059}$ & $\mathbf{16.257}_{.049}$ & $\mathbf{14.402}_{.090}$\\
  \textsc{GW}$+$\textsc{GW}$^8_{\mathrm{out}}$
   & $0.4699_{.0032}$ & $0.6545_{.0033}$ & $\mathbf{0.7653}_{.0030}$
   & $\mathbf{18.200}_{.074}$ & $16.555_{.028}$ & $\mathbf{14.469}_{.042}$\\
  \bottomrule
  \end{tabular}
  \end{adjustbox}
  \end{subtable}
  \vspace{-2mm}
\end{table}
\vspace{-3mm}\paragraph{Full-budget molecular generation}
We train larger GW-aligned endpoint models without outer alignment and evaluate the resulting \textsc{GW-CatFlow} model with 500 Euler
steps. The complete protocol is in Appendix~\ref{app:full-budget-protocol}.
Table~\ref{tab:mol-sota} gives \(99.34\%\) validity and \(0.115\) FCD on QM9,
comparable to DeFoG, and \(99.01\%\) validity with \(0.966\) FCD on ZINC250k,
the lowest FCD listed. \textcolor{black}{This shows that the underlying model remains competitive with state-of-the-art baselines at full budget. The controlled sweep above isolates the effect of coupling.}

\begin{wraptable}{r}{0.57\textwidth}
    \centering
\setlength{\tabcolsep}{1pt}
\caption{Molecular full-budget run. Ours:
$500$ Euler steps, $10{,}000$ samples, mean$_{\rm std}$ over three repeats.
Baselines follow published protocols. GraphAF, MoFlow, and DiGress values are
quoted from \protect\citep{hou2026ggflow}, \textcolor{black}{DeFoG values from
\protect\citep{xiong2026variational}.}}
    \vspace{-2mm}
    \resizebox{\linewidth}{!}{%
\label{tab:mol-sota}
\scriptsize
\begin{tabular}{l
  >{\footnotesize}c >{\footnotesize}c
  >{\footnotesize}c >{\footnotesize}c}
\toprule
& \multicolumn{2}{c}{QM9}
& \multicolumn{2}{c}{ZINC250k}\\
\cmidrule(lr){2-3}\cmidrule(lr){4-5}
{\footnotesize Method}
& Validity ($\uparrow$) & FCD ($\downarrow$)
& Validity ($\uparrow$) & FCD ($\downarrow$)\\
\midrule
GraphAF \smallcite{shi2020graphaf}
& $67.14$ & $5.246$
& $67.92$ & $16.128$\\
MoFlow \smallcite{zang2020moflow}
& $92.03$ & $4.536$
& $63.76$ & $20.875$\\
GDSS \smallcite{jo2022gdss}
& $95.72$ & $2.900$
& $97.01$ & $14.656$\\
CatFlow \smallcite{eijkelboom2024variational}
& $99.81$ & $0.441$
& $99.21$ & $13.211$\\
DiGress \smallcite{vignac2023digress}
& $98.29$ & $0.095$
& $94.98$ & $3.482$\\
GGFlow  \smallcite{hou2026ggflow}
& $99.91$ & $0.148$
& ${99.63}$ & $1.455$\\
DeFoG 
\smallcite{yimingdefog_discrete_fm}
& ${99.30}$ & ${0.120}$
& ${99.22}$ & ${1.425}$\\
VBFN \smallcite{xiong2026variational}
& ${99.98}$ & ${0.083}$
& ${99.63}$ & ${1.307}$\\
\midrule
{\footnotesize \textsc{ GW-CatFlow} (ours)}
  & $99.34_{.08}$ & $0.115_{.004}$
  & $99.01_{.03}$ & ${0.966}_{.013}$\\
\bottomrule
\end{tabular}
}
\end{wraptable}
\vspace{-3.5mm}\section{Conclusion}
\vspace{-2.5mm}
We study inner and outer alignment in equivariant graph flow matching. Both
simplify trajectories and improve few-step generation, while compute scaling
reveals a tradeoff between better inner alignment and larger outer batches.
Faster 
GW approximations
\citep{beier2021linear,chowdhury2021quantized,piening2025novel,jin2022orthogonal,scetbon2022lowrankgw}
could shift this compute tradeoff, while
semi-discrete formulations may allow for outer alignment beyond minibatches
\citep{kong2025alignflow_semidiscre,mousavi2025flow_semidiscre}.
\clearpage

\section*{Acknowledgments}
We thank G. Steidl for supporting the work within her DFG project in the RTG BIOQIC (GRK2260/289347353).

\bibliographystyle{plainnat}
\bibliography{references}
\clearpage
\appendix

\vspace{-3.5mm}\section{Proof of representative lifts and quotient geodesics}
\label{app:measure-transport}
This section proves Theorem~\ref{thm:quotient-lift-geodesic}. We first recall
the metric notions used in the geodesic part of the proof. A curve
$(\mu_t)_{t\in I}$ belongs to
$AC_I^2(\mP_2(X))$ if there exists $m\in L^2(I)$ such that
\[
    \W_{2,X}(\mu_s,\mu_t)
    \leq \int_s^t m(r)\,\d r
    \qquad\text{for all }0\leq s\leq t\leq1.
\]

Let $G \subset O(D)$ be compact and define
\[
    Q \coloneqq \R^D/G,
    \qquad
    [x] \coloneqq \{gx : g \in G\},
\]
with quotient map
\[
    \mathfrak{q} : \R^D \to Q,
    \qquad
    \mathfrak{q}(x)=[x].
\]
The quotient is equipped with the metric
\[
    d_{Q}([x],[y])
    \coloneqq
    \min_{g \in G} \|x-gy\|.
\]
Since $G$ is compact and acts by isometries, the minimum is attained
and $d_Q$ defines a metric on $Q$.  Since $G$ is compact and acts isometrically, $(Q,d_Q)$ is Polish, and $\mathfrak{q}$ is $1$-Lipschitz.

For $\bmu,\bnu\in\mP_2(Q)$, the corresponding Wasserstein distance is
\begin{equation}
\label{eq:quotient_wasserstein}
\W_{2,Q}^2(\bmu,\bnu)
\coloneqq
\min_{\Gamma \in \mathrm{c}_{Q}(\bmu,\bnu)}
\int_{Q \times Q}
d_{Q}(\bx,\by)^2
\,\d\Gamma(\bx,\by),
\end{equation}
where
\[
\mathrm{c}_{Q}(\bmu,\bnu)
\coloneqq
\{ \Gamma \in \mP_2(Q \times Q):
\mathrm{proj}^0_\sharp \Gamma = \bmu,\mathrm{proj}^1_\sharp \Gamma = \bnu\}.
\]
By $\mathrm{c}_{Q}^{{\rm opt}}(\bmu,\bnu)$ we denote the couplings which minimize \eqref{eq:quotient_wasserstein}.

We denote by $\lambda_G$ the normalized Haar probability measure on the compact group $G$.
Thus, for any integrable $f:G\to\R$,
\[
\int_G f(g)\,\d\lambda_G(g)
\]
is invariant under left and right multiplication of the argument of $f$. If $G$ is finite, in particular for permutation groups, $\lambda_G$ is simply the uniform measure,
\[
\int_G f(g)\,\d\lambda_G(g)
=
\frac1{|G|}\sum_{g\in G} f(g).
\]

\begin{proposition}[Representative lifting of quotient couplings] 
\label{prop:representative-lifting}
Let \(\bmu,\bnu\in\mP_2(Q)\), fix \(\mu\in[\bmu]\), and let
\(\Gamma\in\mathrm{c}_Q(\bmu,\bnu)\). Then
\[
\int_{Q\times Q} d_Q(\bx,\by)^2\,\d\Gamma(\bx,\by)
=
\min_{\substack{
\nu\in[\bnu],\,\gamma\in\mathrm{c}_{\R^D}(\mu,\nu)\\
(\mathfrak q,\mathfrak q)_\sharp\gamma=\Gamma}}
\int_{\R^D\times\R^D}\|x-y\|^2\,\d\gamma(x,y).
\]
In particular, if \(\Gamma\in\mathrm{c}_Q^{\rm opt}(\bmu,\bnu)\) and
\((\nu_\Gamma,\gamma_\Gamma)\) attains the minimum on the right-hand side,
then \(\gamma_\Gamma\in\mathrm{c}_{\R^D}^{\rm opt}(\mu,\nu_\Gamma)\) and
\(\W_{2,\R^D}^2(\mu,\nu_\Gamma)=\W_{2,Q}^2(\bmu,\bnu)\).
\end{proposition}

\begin{proof}
Let \(\nu\in[\bnu]\) and
\(\gamma\in\mathrm{c}_{\R^D}(\mu,\nu)\) satisfy
\((\mathfrak q,\mathfrak q)_\sharp\gamma=\Gamma\). Since
\[
    d_Q(\mathfrak q(x),\mathfrak q(y))
    \le \|x-y\|,
\]
we have
\[
    \int_{Q\times Q}d_Q(\bx,\by)^2\,\d\Gamma(\bx,\by)
    \le
    \int_{\R^D\times\R^D}\|x-y\|^2\,\d\gamma(x,y).
\]

Conversely, disintegrate
\[
    \mu=\int_Q\mu_{\bx}\,\d\bmu(\bx),
    \qquad
    \Gamma=\int_Q\delta_{\bx}\otimes\Gamma_{\bx}\,\d\bmu(\bx),
\]
where \(\mu_{\bx}\) is supported on \(\mathfrak q^{-1}(\bx)\), and define
\(\eta\in\mP(\R^D\times Q)\) by
\[
    \eta(A\times B)
    \coloneqq
    \int_Q\mu_{\bx}(A)\Gamma_{\bx}(B)\,\d\bmu(\bx).
\]
Then \((\mathfrak q,\mathrm{id})_\sharp\eta=\Gamma\).

Since \(G\) is compact, the corresponding argmin relation is closed
with nonempty compact sections. Hence, a measurable selection theorem
\citep[Theorem~18.18]{kechris2012classical} yields a Borel map
\(Y(x,\by)\in\mathfrak q^{-1}(\by)\) satisfying
\[
    \|x-Y(x,\by)\|=d_Q(\mathfrak q(x),\by).
\]
Set \(\gamma_\Gamma\coloneqq(\mathrm{id},Y)_\sharp\eta\), and let
\(\nu_\Gamma\) be its second marginal. Then
$
    (\mathfrak q,\mathfrak q)_\sharp\gamma_\Gamma=\Gamma$
    and
    $
    \mathfrak q_\sharp\nu_\Gamma=\bnu.
$
Moreover, since \(G\subset O(D)\), we get
\[
    \int_{\R^D}\|y\|^2\,\d\nu_\Gamma(y)
    =
    \int_Q d_Q(\by,[0])^2\,\d\bnu(\by)<\infty,
\]
and hence \(\nu_\Gamma\in[\bnu]\). Finally,
\[
    \int_{\R^D\times\R^D}\|x-y\|^2\,\d\gamma_\Gamma(x,y)
    =
    \int_{Q\times Q}d_Q(\bx,\by)^2\,\d\Gamma(\bx,\by).
\]
Thus the minimum is attained.
If, in addition, \(\Gamma\in\mathrm{c}_Q^{\rm opt}(\bmu,\bnu)\), then
\[
\W_{2,Q}^2(\bmu,\bnu)
\le \W_{2,\R^D}^2(\mu,\nu_\Gamma)
\le \int_{\R^D\times\R^D}\|x-y\|^2\,\d\gamma_\Gamma(x,y)
= \W_{2,Q}^2(\bmu,\bnu).
\]
Hence all inequalities are equalities, which proves the final claim.
\end{proof}

\begin{proof}[Symmetrized representative claim in
Theorem~\ref{thm:quotient-lift-geodesic}]
Let $\Gamma$ be optimal and let $(\nu_\Gamma,\gamma_\Gamma)$ be as in
Proposition~\ref{prop:representative-lifting}. We have for $\nu_1,\nu_2\in [\bnu]$ that
\begin{align}
    \int_{\mathbb R^D} f(x)\,\d\nu_1^G(x)
&=
\int_Q
\left(
\int_G f(gx)\,\d\lambda_G(g)
\right)
\d(\mathfrak q_\sharp\nu_1)([x])\\
&=\int_Q
\left(
\int_G f(gx)\,\d\lambda_G(g)
\right)
\d(\mathfrak q_\sharp\nu_2)([x])\\
&=\int_{\mathbb R^D} f(x)\,\d\nu_2^G(x)
\end{align}
because $x\mapsto \left(
\int_G f(gx)\,\d\lambda_G(g)
\right)$ is $G$-invariant and descends to a function on the quotient space. Thus all representatives of the same quotient law have the same
$G$-symmetrization,  and we have
$\nu_\Gamma^G=\nu^G$ for every $\nu\in[\bnu]$. 
Define
\[
\gamma_\Gamma^G
\coloneqq
\mathbb E_{g\sim\lambda_G}[(g,g)_\sharp\gamma_\Gamma].
\]
Then
$\gamma_\Gamma^G\in
\mathrm c_{\R^D}(\mu^G,\nu_\Gamma^G)$ and, since $G\subset O(D)$,
\[
\int\|x-y\|^2\,\d\gamma_\Gamma^G
=
\int\|x-y\|^2\,\d\gamma_\Gamma
=
\W_{2,Q}^2(\bmu,\bnu).
\]
Moreover,
$\mathfrak q_\sharp\mu^G=\bmu$ and
$\mathfrak q_\sharp\nu_\Gamma^G=\bnu$, hence
\[
\W_{2,Q}^2(\bmu,\bnu)
\le
\W_{2,\R^D}^2(\mu^G,\nu_\Gamma^G)
\le
\int\|x-y\|^2\,\d\gamma_\Gamma^G.
\]
Thus
\[
\W_{2,\R^D}(\mu^G,\nu^G)
=
\W_{2,\R^D}(\mu^G,\nu_\Gamma^G)
=
\W_{2,Q}(\bmu,\bnu).
\]
\end{proof}

\begin{proof}[Geodesic and energy claims in
Theorem~\ref{thm:quotient-lift-geodesic}]
Let $\Gamma$ be optimal and let $(\nu_\Gamma,\gamma_\Gamma)$ be the pair
constructed in Proposition~\ref{prop:representative-lifting}. The proposition
gives
\begin{equation}
\label{eq:lifted-and-quotient-cost}
\gamma_\Gamma\in\mathrm{c}_{\R^D}^{\rm opt}(\mu,\nu_\Gamma),
\qquad
\W_{2,\R^D}^2(\mu,\nu_\Gamma)
=\W_{2,Q}^2(\bmu,\bnu).
\end{equation}
For $0\leq s<t\leq1$, the coupling
$(\pr^s,\pr^t)_\sharp\gamma_\Gamma$ and the fact that $\mathfrak q$ is
$1$-Lipschitz give
\[
\W_{2,Q}(\bmu_s,\bmu_t)
\leq
\W_{2,\R^D}(\mu_s,\mu_t)
\leq
(t-s)\W_{2,Q}(\bmu,\bnu).
\]
Thus $(\bmu_t)_{t\in I}$ belongs to $AC_I^2(\mP_2(Q))$. Applying this
bound on $[0,s]$, $[s,t]$, and $[t,1]$, and combining it with the triangle
inequality between $\bmu_0=\bmu$ and $\bmu_1=\bnu$, forces equality in each
bound. Hence
\[
\W_{2,Q}(\bmu_s,\bmu_t)
=(t-s)\W_{2,Q}(\bmu,\bnu),
\]
so $(\bmu_t)_{t\in I}$ is a constant-speed geodesic.

Finally, the Euclidean optimal-coupling identity in
Section~\ref{subsec:euclidean-flow-matching} and
\eqref{eq:lifted-and-quotient-cost} yield
\[
\int_I\int_{\R^D}\|v_t^{\gamma_\Gamma}(z)\|^2
\,\d\mu_t(z)\,\d t
=\W_{2,Q}^2(\bmu,\bnu).
\]
\end{proof}

\clearpage\section{Proof of equivariant flow matching via symmetrization}
\label{app:equivariant-theory}
This section proves Theorem~\ref{thm:equivariant-symmetrization}. Write
$\lambda_G$ for the normalized Haar probability measure on $G$. We call
$\pi\in\mathrm{c}_{\R^D}(\mu,\nu)$ diagonally $G$-invariant if
$(g,g)_\sharp\pi=\pi$ for every $g\in G$.

\begin{proposition}[Equivariance of the flow-matching minimizer]
Let $\pi\in\mathrm{c}_{\R^D}(\mu,\nu)$ be diagonally $G$-invariant and let
$v^\pi$ denote the minimizer of $J_\pi$ in \eqref{minvr}. Then $v^\pi$
admits a $G$-equivariant representative. More precisely, there exists a
jointly Borel field $\bar v:I\times\R^D\to\R^D$ such that
\[
\bar v_t=v_t^\pi
\qquad
\mu_t\text{-a.e. for a.e. }t\in I,
\]
and
\[
\bar v_t(gz)=g\bar v_t(z)
\qquad
\text{for every }g\in G,\ z\in\R^D,\ t\in I .
\]
\end{proposition}

\begin{proof}
Let $(X,Y)\sim\pi$, $Z_t\coloneqq\pr^t(X,Y)$ and $U\coloneqq Y-X$.
Diagonal invariance and linearity of the action give
\[
(Z_t,U)\overset{\mathrm d}=(gZ_t,gU)
\qquad\text{for every }g\in G.
\]
In particular, each $\mu_t$ is $G$-invariant. For $g\in G$, define
\[
v_t^{\pi,g}(z)\coloneqq g^{-1}v_t^\pi(gz).
\]
Diagonal invariance and orthogonality give
$J_\pi(v^{\pi,g})=J_\pi(v^\pi)$. By uniqueness of the minimizer in
$L^2(\d t\,\d\mu_t)$,
\[
g^{-1}v_t^\pi(g\,\cdot)=v_t^\pi
\qquad
\d t\,\d\mu_t\text{-a.e.}
\]
for every $g\in G$.  Choose a jointly Borel representative of $v^\pi$ and write
\[
F(t,z,h)\coloneqq h^{-1}v_t^\pi(hz).
\]
The action is continuous, so $F$ is jointly Borel. Orthogonality of $h$ and
$G$-invariance of every $\mu_t$ give
\[
\begin{aligned}
&\int_I\int_{\R^D}\int_G
\|F(t,z,h)\|^2
\,\d\lambda_G(h)\,\d\mu_t(z)\,\d t\\
&\qquad=
\int_G\int_I\int_{\R^D}
\|v_t^\pi(hz)\|^2
\,\d\mu_t(z)\,\d t\,\d\lambda_G(h)
=\|v^\pi\|_{L^2(\d t\,\d\mu_t)}^2<\infty.
\end{aligned}
\]
By Tonelli's theorem, the inner $L^2(G)$ integral is finite for
$\d t\,\d\mu_t$-almost every $(t,z)$. Since $\lambda_G$ is a probability
measure, Cauchy--Schwarz then gives
\[
\int_G\|F(t,z,h)\|\,\d\lambda_G(h)<\infty
\]
for almost every $(t,z)$. Hence the vector-valued Haar integral exists outside
a $\d t\,\d\mu_t$-null set. Define
\[
\bar v_t(z)\coloneqq
\int_G F(t,z,h)\,\d\lambda_G(h)
\]
on this set and set it to zero otherwise. Componentwise integration gives a
jointly Borel field $\bar v$.

It remains to show that this field represents the same $L^2$ minimizer. Since
the preceding equality holds $\d t\,\d\mu_t$-almost everywhere for every
fixed $h\in G$, another application of Tonelli's theorem gives
\[
\int_I\int_{\R^D}\int_G
\bigl\|F(t,z,h)-v_t^\pi(z)\bigr\|^2
\,\d\lambda_G(h)\,\d\mu_t(z)\,\d t=0.
\]
Consequently, for $\d t\,\d\mu_t$-almost every $(t,z)$, the integrand
vanishes for $\lambda_G$-almost every $h$, and hence
$\bar v_t(z)=v_t^\pi(z)$. Thus the averaging changes only the representative
of the $L^2(\d t\,\d\mu_t)$ minimizer.

It remains to check that the averaged representative is exactly equivariant.
For $k\in G$, right invariance of Haar measure and the change of variables
$r=hk$ give
\[
\begin{aligned}
\bar v_t(kz)
&=\int_G h^{-1}v_t^\pi(hkz)\,\d\lambda_G(h)\\
&=\int_G k r^{-1}v_t^\pi(rz)\,\d\lambda_G(r)
=k\bar v_t(z).
\end{aligned}
\]
The set where the Haar integral is not finite is itself $G$-invariant by the
same change of variables. Therefore the equality also holds there under the
zero convention, and $\bar v_t(kz)=k\bar v_t(z)$ for every
$k\in G$, $z\in\R^D$, and $t\in I$.

\end{proof}
For \(\eta\in\mP_2(\R^D)\) and
\(\pi\in\mathrm{c}_{\R^D}(\mu,\nu)\), define
\[
\eta^G\coloneqq\int_G g_\sharp\eta\,\d\lambda_G(g),
\qquad
\pi^G\coloneqq\int_G(g,g)_\sharp\pi\,\d\lambda_G(g),
\]
Then
\[
\pi^G\in\mathrm{c}_{\R^D}(\mu^G,\nu^G),
\qquad
(\mathfrak q,\mathfrak q)_\sharp\pi^G
=
(\mathfrak q,\mathfrak q)_\sharp\pi.
\]
Haar invariance gives $(h,h)_\sharp\pi^G=\pi^G$ for every $h\in G$.
Thus diagonal symmetrization preserves the coupling on the quotient, and
\(\pi^G\in\mathrm{c}_{\R^D}(\mu,\nu)\) whenever both marginals are
\(G\)-invariant.

\begin{proposition}[Minimization over equivariant fields]
Let $\pi\in\mathrm{c}_{\R^D}(\mu,\nu)$ be arbitrary. Then the
field $v^{\pi^G}$ associated with $\pi^G$ minimizes $J_\pi$ over all
$G$-equivariant fields:
\[
v^{\pi^G}\in
\argmin_{\substack{v:I\times\R^D\to\R^D\ \mathrm{jointly\ Borel}\\
v_t(gx)=gv_t(x)\ \forall\,t,g,x}}
J_\pi(v).
\]
In particular, the constrained problem is the ordinary
flow-matching problem for the diagonally symmetrized coupling $\pi^G$.
\end{proposition}

\begin{proof}
For every $G$-equivariant field $v$ and every $g\in G$, orthogonality
of the action gives
\[
J_{(g,g)_\sharp\pi}(v)
=
\int_I\int_{\R^D\times\R^D}
\bigl\|v_t(g\pr^t(x,y))-g(y-x)\bigr\|^2
\,\d\pi(x,y)\,\d t
=
J_\pi(v).
\]
Averaging over $G$ therefore yields $J_{\pi^G}(v)=J_\pi(v)$ for every
equivariant $v$. Since $\pi^G$ is diagonally $G$-invariant, its
unrestricted minimizer $v^{\pi^G}$ admits a $G$-equivariant
representative. As $v^{\pi^G}$ minimizes $J_{\pi^G}$ over all fields,
it also minimizes $J_\pi$ over the equivariant ones.
\end{proof}
\begin{proposition}[Equivariant flows on the quotient]
\label{prop:equivariant-flow-symmetrization}
Let \(v:I\times\R^D\to\R^D\) be \(G\)-equivariant and assume that
\eqref{flow-ode} admits a unique flow \(\gamma_t\). Then \(\gamma_t\) induces
the well-defined quotient flow
\[
\bar\gamma_t([x])\coloneqq[\gamma_t(x)].
\]
For every \(\eta\in\mP_2(\R^D)\) and \(t\in I\),
\[
\mathfrak q_\sharp\gamma_{t,\sharp}\eta
=
\bar\gamma_{t,\sharp}\mathfrak q_\sharp\eta.
\]
In particular, if
\(\mathfrak q_\sharp\eta=\mathfrak q_\sharp\widetilde\eta\), then
\[
\mathfrak q_\sharp\gamma_{t,\sharp}\eta
=
\mathfrak q_\sharp\gamma_{t,\sharp}\widetilde\eta,
\]
so the projected curve is independent of the representative lift at time
\(0\).
\end{proposition}

\begin{proof}
Equivariance of \(v\) and uniqueness of the ODE imply
\(\gamma_t(gx)=g\gamma_t(x)\). Thus \(\bar\gamma_t\) is well-defined and
\(\mathfrak q\circ\gamma_t=\bar\gamma_t\circ\mathfrak q\), which gives
\[
\mathfrak q_\sharp\gamma_{t,\sharp}\eta
=
\bar\gamma_{t,\sharp}\mathfrak q_\sharp\eta.
\]
The final claim follows by applying this identity to two measures with the
same quotient push-forward.
\end{proof}

The coordinatewise categorical formulation in
Section~\ref{subsec:endpoint-euclidean-flow-matching} is directly compatible
with coordinate permutation actions.

\begin{corollary}[Coordinatewise equivariant endpoint prediction]
\label{cor:coordinatewise-equivariant-endpoint-prediction}
\textcolor{black}{Assume that \(G\) consists of coordinate-permutation matrices}, where
\(g\cdot d\) is determined by \((gx)_{g\cdot d}=x_d\), and let
\(
\mathcal A=\{a_1,\ldots,a_M\}~\subset~\R.
\)
Let \(\mu,\nu\in\mP_2(\R^D)\) be arbitrary, assume that \(\nu\) is supported
on \(\mathcal A^D\), and let \(\pi\in\mathrm{c}_{\R^D}(\mu,\nu)\). Then the
conditional coordinate probabilities \(q^{\ast,\pi^G,d}\) associated
with \(\pi^G\) admit a representative satisfying
\[
q_t^{\ast,\pi^G,g\cdot d}(n\mid gz)
=
q_t^{\ast,\pi^G,d}(n\mid z)
\]
for every \(t\in I\), \(g\in G\), \(d\in[D]\), \(n\in[M]\), and
\(z\in\R^D\), and minimize
\( J_\pi^{\rm coord}\) over all such equivariant coordinate
kernels:
\[
\bigl(q^{\ast,\pi^G,d}\bigr)_{d=1}^D
\in
\argmin_{\substack{q^d:I\times\R^D\to\Delta_M\ \mathrm{jointly\ Borel}\\
q_t^{g\cdot d}(n\mid gz)=q_t^d(n\mid z)}}
 J_\pi^{\rm coord}\bigl((q^d)_{d=1}^D\bigr).
\]
Writing \(\mu_t^G\coloneqq\pr^t_\sharp\pi^G\), for a.e.\ \(t<1\) and
\(\mu_t^G\)-a.e.\ \(z\),
\[
\bigl(v_t^{\pi^G}(z)\bigr)_d
=
\frac{1}{1-t}
\left(
\sum_{n=1}^M a_nq_t^{\ast,\pi^G,d}(n\mid z)-z_d
\right).
\]
In particular, this defines a \(G\)-equivariant representative of
\(v^{\pi^G}\), without requiring \(\mu\) or \(\nu\) to be
\(G\)-invariant.
\end{corollary}

\begin{proof}
\textcolor{black}{Since \(G\) consists of coordinate-permutation matrices, it is
finite.} \textcolor{black}{Diagonal} invariance of $\pi^G$ allows us to choose the conditional
coordinate probabilities equivariantly.
If $(X,Y)\sim\pi^G$ and $Z_t=\pr^t(X,Y)$, then
$(Z_t,Y)\overset{\mathrm d}=(gZ_t,gY)$ for every $g\in G$. Consequently,
starting from any jointly Borel version $q$, each
$q_t^{g\cdot d}(n\mid gz)$ is a version of the same conditional probability.
Define
\[
\bar q_t^d(n\mid z)
\coloneqq
\frac{1}{|G|}\sum_{g\in G}q_t^{g\cdot d}(n\mid gz).
\]
This is another version of the same conditional probabilities and satisfies
the displayed equivariance identity for every $g,d,n,t,z$. We use this
version below.

For every equivariant coordinate kernel \(q\), relabelling the coordinates
and summing over \(d\in[D]\) gives
\[
 J_{(g,g)_\sharp\pi}^{\rm coord}(q)
=
 J_\pi^{\rm coord}(q).
\]
Hence
\( J_{\pi^G}^{\rm coord}(q)
= J_\pi^{\rm coord}(q)\) on the equivariant class. By
\eqref{eq:euclidean-fm-coordinate-cross-entropy}, these probabilities
minimize the unrestricted coordinatewise objective for \(\pi^G\), and
therefore the constrained objective for \(\pi\).

For any equivariant coordinate kernel \(q\), define
\[
\textcolor{black}{\bigl(v_t^q(z)\bigr)_d}
\coloneqq
\frac{1}{1-t}
\left(
\sum_{n=1}^M a_nq_t^d(n\mid z)-z_d
\right).
\]
Then
\[
\textcolor{black}{\bigl(v_t^q(gz)\bigr)_{g\cdot d}}
=
\frac{\sum_{n=1}^M
a_nq_t^{g\cdot d}(n\mid gz)-(gz)_{g\cdot d}}{1-t}
=
\textcolor{black}{\bigl(v_t^q(z)\bigr)_d}.
\]
Thus \(v_t^q(gz)=gv_t^q(z)\), independently of the invariance of the
endpoint measures. Taking
\(q^d=q^{\ast,\pi^G,d}\) for every \(d\in[D]\) and using
\eqref{eq:euclidean-fm-coordinate-endpoint-mixture} proves the remaining
claims.
\end{proof}
\begin{remark}[Categorical feature blocks]
\label{rem:categorical_feature_blocks}
The scalar formulation $\mathcal A\subset\R$ above is not restrictive for
finite categorical data, but a scalar encoding introduces an arbitrary
geometry between classes. In practice, we therefore use one-hot encodings
and partition the coordinates into categorical feature blocks
$b\in\mathcal B$. Let block $b$ have $K_b$ possible values
$\{e_1,\ldots,e_{K_b}\}\subset\R^{K_b}$, where $e_k$ denotes the $k$-th
standard basis vector. Writing $X'_b$ and $Z_{t,b}$ for the corresponding
target and interpolated blocks, respectively, define
\[
q_t^{\ast,b}(k\mid z)
\coloneqq
\mathbb P_\pi\!\left(X'_b=e_k\mid Z_t=z\right),
\qquad k\in[K_b].
\]
A softmax head estimates $q_t^{\ast,b}(\cdot\mid z)$ and is trained with
the blockwise cross-entropy
\[
J_\pi^{\rm block}(q)
\coloneqq
-\int_I\int
\sum_{b\in\mathcal B}
\log q_t^b\!\left(\kappa_b(x')\mid (1-t)x+tx'\right)
\,\d\pi(x,x')\,\d t,
\]
where $X'_b=e_{\kappa_b(X')}$. Its conditional endpoint mean is
\[
m_t^{\ast,b}(z)
\coloneqq
\mathbb E_\pi[X'_b\mid Z_t=z]
=
\sum_{k=1}^{K_b} e_k\,q_t^{\ast,b}(k\mid z),
\]
and hence, for a.e.\ $t<1$ and $\mu_t$-a.e.\ $z$, the corresponding
velocity block is
\[
\bigl(v_t^\pi(z)\bigr)_b
=
\frac{m_t^{\ast,b}(z)-z_b}{1-t}
=
\frac{1}{1-t}
\left(
\sum_{k=1}^{K_b} e_k\,q_t^{\ast,b}(k\mid z)-z_b
\right).
\]
Thus different feature types may use blocks of different cardinalities.
For molecular graphs, for example, edge type, atom type, and formal charge
form separate categorical blocks. Node relabeling permutes blocks within
each feature type, so the construction applies equivariantly to one-hot
node and edge features without requiring conditional independence between
blocks.
\end{remark}

\clearpage
\section{Related Work on Transport-Based Graph Generation}
\label{supp_sec:related_work}
Graph generators commonly enforce node-relabeling symmetry through equivariant
architectures \citep{jo2022gdss,vignac2023digress,eijkelboom2024variational}.
Recent graph flow models additionally exploit transport geometry. GGFlow uses
minibatch OT for graph pairing \citep{hou2026ggflow}, while BWFlow constructs
paths from Bures--Wasserstein transport between graph representations
\citep{jiang2025bureswasserstein_flow_matching}. Concurrent Flowette introduces
a graph-structured graphette prior and uses fused Gromov--Wasserstein distances
for structure-aware minibatch pairing \citep{wijesinghe2026flowette}.
Both GGFlow and Flowette therefore use transport primarily to choose which
source and target graphs to pair. We instead study transport at both levels of graph alignment. Inner alignment
selects node representatives before interpolation, using either a projected GW
plan or Best-of-Augmentation (BoA) over sampled permutations, while outer
alignment assigns graphs within a minibatch. Our experiments show that inner
alignment can provide substantial gains on its own, with BoA effective for
small graphs and GW for larger alignment spaces, while the benefit of outer
alignment depends on the dataset and compute budget. This connects our study
to GW-based graph matching
\citep{peyre2016gromov,vayer2020fused,xu2019gromov}, while deriving both
alignment levels from transport on the permutation quotient.

\clearpage
\vspace{-3.5mm}\section{Continuous target via stochastic block models}
\label{sbm:appendix_sec}
Here, we aim to provide further insight into the advantage of GW-type alignment via a synthetic experiment with a multimodal target distribution.
Therefore, we consider not only \textcolor{black}{typical graph diagnostics} but also GW-type two-sample test statistics for our target samples. To further investigate the role of graph alignment, we additionally test a cheap GW surrogate to see if coarse alignment can already lead to improvements.

\subsection{Cheap GW Surrogate via Lower Bound}
As a cheaper alternative to GW, let
\(\operatorname{ecc}_E(i)=(N^{-1}\sum_k\|E_{ik}\|^2)^{1/2}\), where
\(\operatorname{ecc}_E^\uparrow(i)\) is the \(i\)-th smallest value among
 eccentricities of \(E\). This gives the first lower
bound (\textsc{FLB}) \citep{bauer2025zgw}
\[
\operatorname{FLB}_{2,Z}^2([E],[F])
\coloneqq\frac1N\sum_i
|\operatorname{ecc}_E^\uparrow(i)-\operatorname{ecc}_F^\uparrow(i)|^2
\le\operatorname{GW}_{2,Z}^2([E],[F]).
\]
This is the squared \(2\)-Wasserstein distance between the empirical
eccentricity distributions. We sort the nodes of \(E\) and \(F\) by their
eccentricities and pair the \(i\)-th node in one ordering with the \(i\)-th node
in the other, breaking ties arbitrarily. These pairs define a permutation. Computing the eccentricities takes
\(O(N^2C)\) operations, and sorting takes \(O(N\log N)\). Thus \textsc{GW}
retains detailed pairwise structure at higher computational cost, whereas
\textsc{FLB} provides an inexpensive inner alignment and is cheap to
evaluate for all \(B^2\) candidate pairs in the outer minibatch assignment.
Their empirical cost--quality trade-off is reported in
Appendix~\ref{app:solver-cost}. Again, we can use the FLB not only for inner alignment of nodes (\textsc{FLB}), but we can also reorder the matching of source and target via FLB as a cost  (\textsc{FLB}$+$\textsc{FLB}$^B_{\mathrm{out}}$). Again, \textsc{FLB}$+$\textsc{FLB}$^8_{\mathrm{out}}$ reorders
same-size graphs in sub-batches of at most eight. 
\subsection{Unconditional continuous-SBM experiment}
\label{subsec:sbm}

\vspace{-1mm}\paragraph{Data}
As a controlled continuous testbed, we use an SBM on \(N=10\) nodes with two
channels: one symmetric weighted-edge channel in \([0,1]\) and one scalar node
feature on the diagonal. For \(K\in\{1,\ldots,5\}\), nodes are assigned to
near-balanced blocks; for every \(i<j\), we sample and mirror
\(e^E_{ij}=e^E_{ji}\) from \(\mathrm{Beta}(6,2)\) within a block and
\(\mathrm{Beta}(2,6)\) across blocks. A node in block
\(k\in\{0,\ldots,K-1\}\) receives a Beta
feature with concentration \(8\) and mean
\(\mu_k=(2k+1)/(2K)\), so the diagonal channel also records community
structure. We randomize node order, making each tensor an arbitrary
representative of its graph orbit. The source has independent
\(\mathcal U[0,1]\) upper-triangular edges, mirrored across the diagonal, and
independent \(\mathcal U[0,1]\) node features. 
 Our main experiment pools all
\(K\). Appendix~\ref{app:conditional-sbm} instead supplies \(K\) to the model to showcase conditional generation.

\vspace{-3mm}\paragraph{Metrics}
For 100 real and 100 generated graphs, we estimate maximum mean discrepancies (MMDs) of degree, clustering, and
four-node graphlet-orbit counts, which distinguish the roles a node can occupy
in induced four-node subgraphs \citep{jo2022gdss,you2018graphrnn}. 
All three require unweighted graphs, obtained by thresholding real and
generated edges alike at $e^E_{ij}>1/2$.
To compare
the full attributed graphs including their node features, we
further report fused-GW nearest-neighbor accuracy (FGW--NNA).
This statistic builds on the popular OT-NNA statistic used to compare real and generated sets of point clouds
\citep{yang2019pointflow_ot_nn} by replacing pairwise Wasserstein distance with pairwise FGW. Given 100 generated and
100 real graphs, we pool both samples and label each graph by its origin. We then
classify every graph by the label of its nearest neighbor under fused GW,
excluding the graph itself. Because fused GW compares attributed graphs up to
node relabelling, this is a two-sample test directly on graph orbits rather
than on selected descriptors. If real and generated graphs follow the same
law, the accuracy approaches \(0.5\). Larger values mean that the samples are
more easily separated. Results are reported over three training runs with different seeds and evaluated
at \(5/25/125\) Euler steps.
\vspace{-3mm}\paragraph{Results}
At five steps (Table~\ref{tab:sbm}), GW with outer assignment reduces FGW--NNA
from \(0.796\) to \(0.568\), degree MMD from \(0.114\) to \(0.018\), and
clustering MMD from \(0.179\) to \(0.034\). The gap narrows by 125 steps, as
expected because all matching should result in valid interpolations. FLB
captures part of the low-step gain at much lower cost, while outer assignment
adds a smaller improvement.
The permutation-blind \textsc{MBOT} already matches or exceeds the
\textsc{FLB} variants on most statistics, yet remains clearly behind the GW
variants.
The conditional experiment likewise favors GW alignment (Appendix~\ref{app:conditional-sbm}). Notably, at five steps the two
GW variants are also the only conditions that bring FGW--NNA close to its ideal
value of \(0.5\). Thus, the improvement is visible not merely in selected graph
statistics, but in a two-sample test using the underlying relational geometry.

\begin{table}[H]
\centering
\footnotesize
\setlength{\tabcolsep}{3pt}
\caption{Continuous SBM, \emph{unconditional} mixture over $K\in\{1,\dots,5\}$
($N=10$), at $5/25/125$ Euler steps. Rows are mean$_{\text{std}}$ over $3$
training seeds ($5$ evaluation repeats each). Best
value per column is bold (MMDs (degree,
clustering, graphlet orbit) $\downarrow$, Best FGW--NNA: $\sim 0.5$). }
\label{tab:sbm}
\resizebox{\linewidth}{!}{%
\begin{tabular}{l ccc ccc ccc ccc}
\toprule
& \multicolumn{3}{c}{FGW--NNA ($\to 0.5$)}
& \multicolumn{3}{c}{Degree MMD ($\downarrow$)}
& \multicolumn{3}{c}{Clustering MMD ($\downarrow$)}
& \multicolumn{3}{c}{Graphlet-orbit MMD ($\downarrow$)}\\
\cmidrule(lr){2-4}\cmidrule(lr){5-7}\cmidrule(lr){8-10}\cmidrule(lr){11-13}
Method & $5$ & $25$ & $125$ & $5$ & $25$ & $125$ & $5$ & $25$ & $125$ & $5$ & $25$ & $125$\\
\midrule
\textsc{Random}
 & $0.796_{.012}$ & $0.581_{.021}$ & $0.561_{.015}$
 & $0.114_{.045}$ & $0.067_{.016}$ & $0.059_{.011}$
 & $0.179_{.100}$ & $0.085_{.042}$ & $0.075_{.028}$
 & $0.062_{.033}$ & $0.046_{.033}$ & $0.043_{.025}$\\
 \textsc{MBOT}$^{16}$
 & $0.687_{.018}$ & $0.561_{.016}$ & $0.530_{.007}$
 & $0.031_{.004}$ & $0.023_{.006}$ & $0.015_{.005}$
 & $0.090_{.014}$ & $0.046_{.010}$ & $0.037_{.010}$
 & $0.006_{.001}$ & $0.011_{.001}$ & $0.007_{.000}$\\
\textsc{FLB}
 & $0.725_{.013}$ & $0.558_{.011}$ & $0.555_{.022}$
 & $0.054_{.004}$ & $0.037_{.011}$ & $0.030_{.010}$
 & $0.109_{.037}$ & $0.042_{.006}$ & $0.040_{.011}$
 & $0.030_{.014}$ & $0.027_{.017}$ & $0.021_{.015}$\\
\textsc{FLB}$+$\textsc{FLB}$^8_{\mathrm{out}}$
 & $0.691_{.040}$ & $0.562_{.019}$ & $0.547_{.009}$
 & $0.034_{.011}$ & $\mathbf{0.012}_{.002}$ & $0.013_{.007}$
 & $0.121_{.041}$ & $0.045_{.024}$ & $0.047_{.021}$
 & $0.018_{.011}$ & $\mathbf{0.007}_{.002}$ & $0.010_{.005}$\\
\textsc{GW}
 & $0.577_{.011}$ & $\mathbf{0.522}_{.004}$ & $0.520_{.023}$
 & $0.027_{.013}$ & $0.014_{.002}$ & $0.012_{.004}$
 & $0.049_{.019}$ & $\mathbf{0.015}_{.005}$ & $0.016_{.007}$
 & $0.014_{.004}$ & $0.010_{.003}$ & $0.008_{.001}$\\
\textsc{GW}$+$\textsc{GW}$^8_{\mathrm{out}}$
 & $\mathbf{0.568}_{.008}$ & $0.536_{.014}$ & $\mathbf{0.504}_{.021}$
 & $\mathbf{0.018}_{.005}$ & $0.014_{.002}$ & $\mathbf{0.007}_{.001}$
 & $\mathbf{0.034}_{.007}$ & $0.021_{.004}$ & $\mathbf{0.014}_{.002}$
 & $\mathbf{0.005}_{.001}$ & $0.008_{.001}$ & $\mathbf{0.004}_{.000}$\\
\bottomrule
\end{tabular}
}
\end{table}
\subsection{Conditional continuous-SBM experiment}
\label{app:conditional-sbm}
We additionally condition the continuous SBM model on the community count
\(K\), form minibatches containing only one value of \(K\), and compare pooled generated and real
graphs with matching \(K\)-multisets. All other settings and metrics agree with
Section~\ref{subsec:sbm}. The conditional results in
Table~\ref{tab:cond-sbm} reproduce the unconditional ranking: GW alignment has
the largest advantage at five Euler steps, and the gap narrows with the
integration budget.

\begin{table}[h]
\centering
\footnotesize
\setlength{\tabcolsep}{3pt}
\caption{Continuous SBM, \emph{class-conditional} on
$K\in\{1,\dots,5\}$ ($N=10$), evaluated by pooling generated and real graphs
with the same ground-truth $K$ multiset. Mean $\pm$ std over $3$ training seeds
($5$ evaluation repeats each) at $5/25/125$ Euler steps. Descriptor MMDs:
lower is better, FGW--NNA: closest to $0.5$ is better. Best per column in bold.}
\label{tab:cond-sbm}
\resizebox{\linewidth}{!}{%
\begin{tabular}{l ccc ccc ccc ccc}
\toprule
& \multicolumn{3}{c}{FGW--NNA ($\to 0.5$)}
& \multicolumn{3}{c}{Degree MMD ($\downarrow$)}
& \multicolumn{3}{c}{Clustering MMD ($\downarrow$)}
& \multicolumn{3}{c}{Graphlet-orbit MMD ($\downarrow$)}\\
\cmidrule(lr){2-4}\cmidrule(lr){5-7}\cmidrule(lr){8-10}\cmidrule(lr){11-13}
Method & $5$ & $25$ & $125$ & $5$ & $25$ & $125$ & $5$ & $25$ & $125$ & $5$ & $25$ & $125$\\
\midrule
\textsc{Random}
 & $0.746_{.009}$ & $0.580_{.008}$ & $0.544_{.021}$
 & $0.043_{.009}$ & $0.008_{.002}$ & $0.005_{.001}$
 & $0.076_{.010}$ & $0.029_{.007}$ & $0.020_{.004}$
 & $0.028_{.002}$ & $0.006_{.004}$ & $0.003_{.003}$\\
\textsc{FLB}
 & $0.714_{.010}$ & $0.543_{.017}$ & $0.558_{.023}$
 & $0.023_{.004}$ & $0.005_{.001}$ & $0.005_{.001}$
 & $0.047_{.006}$ & $0.021_{.003}$ & $0.017_{.003}$
 & $0.018_{.003}$ & $0.005_{.001}$ & $0.003_{.002}$\\
\textsc{FLB}$+$\textsc{FLB}$^8_{\mathrm{out}}$
 & $0.699_{.007}$ & $0.552_{.018}$ & $0.539_{.007}$
 & $0.013_{.002}$ & $0.007_{.002}$ & $0.007_{.001}$
 & $0.049_{.009}$ & $0.030_{.005}$ & $0.027_{.010}$
 & $\mathbf{0.009}_{.006}$ & $0.005_{.003}$ & $0.005_{.004}$\\
\textsc{GW}
 & $0.584_{.002}$ & $\mathbf{0.523}_{.010}$ & $\mathbf{0.524}_{.003}$
 & $0.011_{.001}$ & $\mathbf{0.005}_{.001}$ & $\mathbf{0.004}_{.001}$
 & $\mathbf{0.016}_{.002}$ & $\mathbf{0.011}_{.001}$ & $\mathbf{0.010}_{.001}$
 & $0.016_{.001}$ & $0.005_{.002}$ & $0.003_{.001}$\\
\textsc{GW}$+$\textsc{GW}$^8_{\mathrm{out}}$
 & $\mathbf{0.561}_{.017}$ & $0.526_{.028}$ & $0.529_{.014}$
 & $\mathbf{0.008}_{.001}$ & $0.005_{.001}$ & $0.005_{.002}$
 & $0.018_{.001}$ & $0.014_{.001}$ & $0.014_{.004}$
 & $0.012_{.001}$ & $\mathbf{0.004}_{.001}$ & $\mathbf{0.003}_{.001}$\\
\bottomrule
\end{tabular}
}
\end{table}

\clearpage\section{Full-budget molecular training protocol}
\label{app:full-budget-protocol}
  \vspace{-3mm}\paragraph{Protocol}
  For the full-budget evaluation, we retain the categorical endpoint objective
  and use GW inner alignment with independent outer pairing. We increase the
  model capacities and training budgets toward those used by DeFoG
  \citep{yimingdefog_discrete_fm}. In particular, we augment the transformer inputs with
  structural graph statistics and relative random-walk probabilities (RRWP), and
  use self-conditioning.  For our self-conditioning implementation, on \(50\%\) of optimization steps, the model is
  conditioned on a detached preliminary endpoint prediction
  \citep{chen2023analogbits}. During sampling, each Euler step receives the
  preceding step's endpoint prediction. As regularization during training, Gaussian noise with standard deviation
  \(0.5\,t(1-t)\) is added to the interpolated network input while leaving the
  endpoint target unchanged.

  The categorical objective consists of a bond cross-entropy term and a
  node-feature term, the latter averaging the atom-type and formal-charge
  cross-entropies. These terms are weighted equally for QM9. For ZINC250k, we
  multiply the bond term by \(5\) relative to the node-feature term, following
  the stronger emphasis on edge prediction used by DeFoG and CatFlow.

  Both QM9 and ZINC250k are trained in two stages. We first train  without dropout,
  then restore the model and EMA weights and fine-tune with
  dropout \(0.1\) using
  a new AdamW optimizer. All reported full-budget numbers in this  section use the
  final EMA checkpoint after dropout fine-tuning. The \textcolor{black}{dropout-free} checkpoints are
  only used to initialize the second stage.
  
   The QM9 model has node, edge, and global widths \(256/64/64\),
feed-forward
  widths \(256/128/128\), depth \(9\), \(8\) attention heads, and  \(12\) RRWP
  channels, for approximately \(6.0\)M parameters. It is first
  trained for
  \(1000\) epochs with batch size \(1024\) without dropout, and
  then fine-tuned
  for \(50\) additional epochs with dropout \(0.1\). The ZINC250k  model has widths
  \(256/64/128\), feed-forward widths \(256/128/256\), depth
  \(12\), \(8\) heads,
  and \(20\) RRWP channels, for approximately \(10.8\)M
  parameters. It is first
  trained for \(300\) epochs using microbatches of \(128\) with
  two-step gradient
  accumulation, giving an effective batch size of \(256\), and
  then fine-tuned for
  \(30\) additional epochs with dropout \(0.1\).
  The initial training runs use AdamW with learning rate
  \(2\times10^{-4}\),
  weight decay \(10^{-4}\), gradient-norm clipping at \(1.0\),
  cosine
  learning-rate decay, and EMA decay \(0.999\). 
  
  The dropout fine-tuning stages
  restore both model and EMA weights, use a new AdamW optimizer
  at learning rate
  \(5\times10^{-5}\), and do not use cosine decay. We report only  the final
  fine-tuned EMA checkpoints, evaluated using \(500\) Euler
  steps.

\vspace{-3mm}\paragraph{ZINC250k ablation.}
  \begin{table}[b]
  \centering
  \caption{Cumulative ZINC250k ablation at \(25\) Euler steps and reduced training budget. Higher validity is
  better and lower FCD is better.}
  \label{tab:zinc-ablation}
  \begin{tabular}{lcc}
  \toprule
  Setting & Validity & FCD \\
  \midrule
  Base + RRWP & \(0.8846 \pm 0.0030\) & \(12.3371 \pm 0.1229\) \\
  + latent noise & \(0.9080 \pm 0.0012\) & \(11.2470 \pm 0.0704\) \\
  + self-conditioning & \(0.9307 \pm 0.0010\) & \(8.8572 \pm 0.0763\) \\
  + edge weight \(5\) & \(0.9606 \pm 0.0019\) & \(6.9861 \pm 0.0231\) \\
  \bottomrule
  \end{tabular}
  \end{table}

  To isolate the contribution of the main additions used in the full-budget ZINC
  model, we run the cumulative ablation in Table~\ref{tab:zinc-ablation} on
  ZINC250k using the smaller 100-epoch architecture from the base molecular
  experiments. This model has node, edge, and global widths \(128/64/128\), depth
  \(6\), \(8\) attention heads, and about \(2.8\)M parameters. All rows use GW
  inner alignment, independent outer pairing, formal-charge features, the
  categorical endpoint objective, batch size \(32\), and are trained for \(100\)
  epochs. We evaluate final EMA checkpoints with the standard molecular protocol
  using \(10{,}000\) samples and \(3\) repeats, but only at \(25\) Euler steps. The expanded computational budget for the full run \textcolor{black}{may contribute to} the remaining performance gap.

\clearpage\section{Reproducibility details}
\label{app:repro}

This appendix collects the architecture, data, noise, solver, and metric
specifications needed to reproduce the experiments of
Section~\ref{sec:experiments}. All coupling conditions within an experiment share every model and optimization
setting listed here and differ only in their training coupling, namely the inner
solver $\widehat\sigma_a$ and outer reordering $\widehat\tau$ of
Section~\ref{sec:graph-couplings}, with MBOT replacing both by its
permutation-blind minibatch assignment.

\vspace{-3mm}\paragraph{Architecture and optimization}
Every model builds on the same $G_N^{\mathrm{graph}}$-equivariant graph transformer
with node ($X$), edge ($E$), and global ($y$) streams, following the XEy
architecture of \citet{vignac2023digress}. Its node, edge, and global widths
are $d_X=128$, $d_E=64$, and $d_y=128$, respectively, with feed-forward widths
$256/128/256$, depth $6$, $8$ attention heads, and
$\sim\!2.8$M parameters. Training uses AdamW (weight decay $10^{-4}$),
gradient-norm clipping at $1.0$, a cosine learning-rate schedule with base rate
$2\times10^{-4}$, and an exponential moving average of the weights with decay
$0.999$. The EMA weights are used for all evaluation and checkpointing. The
per-experiment channel count $C$, node count $N$, batch size, epoch budget, and
training head are summarized in Table~\ref{tab:hparams}. The SBM models use an
MSE velocity loss \eqref{eq:minibatch-fm-loss}. \textcolor{black}{The molecular models use blockwise categorical endpoint heads
with the cross-entropy objective \eqref{eq:minibatch-endpoint-loss} and recover
the velocity from blockwise conditional means as described in
Appendix~\ref{app:equivariant-theory}.} All experiments use at most a single NVIDIA GeForce RTX 5090 GPU with 32 GB of VRAM.

\vspace{-3mm}\paragraph{Few-step distillation}
Each flow-map student uses the backbone above with an additional start-time embedding,
passed through its own linear layer and added to the global-stream initialization
($8{,}320$ parameters). This layer is zero-initialized and all other weights are copied from
the frozen EMA teacher, so the student initially coincides with the teacher. Pairs $(x,x')$
are drawn independently for every coupling condition, so conditions differ only through the
teacher. $\rho_0$ is the logit-normal distribution $\mathrm{LogitNormal}(-0.4,1)$, i.e.\ the law of
$1/(1+e^{-\xi})$ with $\xi\sim\mathcal N(-0.4,1)$. Under $\rho$,
$t\sim\rho_0$, $s\mid t\sim\mathcal U(0,t)$, and $u=(s+t)/2$, except that a quarter of the
samples use $(s,t)=(0,1)$. Each minibatch is split $75\%/25\%$ between the diagonal and
composition terms of \eqref{eq:flow_map_loss}. Each term is averaged over its part and the two
are summed, with KL terms implemented as soft cross-entropies. Training runs for $100$ epochs
with AdamW (learning rate $10^{-4}$, weight decay $10^{-12}$), gradient-norm clipping at $1.0$,
a cosine schedule, and EMA $0.999$, at batch size $512$ (QM9) and $256$ (ZINC250k). Sampling
uses the uniform grid $t_k=k/K$ with $K\in\{1,2,5\}$.

\begin{table}[b]
\centering
\footnotesize
\setlength{\tabcolsep}{6pt}
\caption{Per-experiment settings for the controlled coupling comparisons,
shared across all coupling conditions.
Architecture, optimizer, EMA, gradient clipping, and learning-rate schedule are
identical throughout (see text), where only the entries below vary. The larger-model
protocol is specified separately in Appendix~\ref{app:full-budget-protocol}.
``Head'' is the
flow-matching parametrization: MSE on the velocity field, or cross-entropy on
the predicted categorical endpoint.}
\label{tab:hparams}
\begin{tabular}{l cccc}
\toprule
& Cycle graphs & Continuous SBM & QM9 & ZINC250k\\
\midrule
Channels $C$        & $3$    & $2$    & $11$    & $16$\\
Nodes $N$           & $12$   & $10$   & $\le 9$ & $\le 38$\\
Gradient batch  ($\neq B$)       & $16$   & $16$   & $64$    & $32$\\
Epochs              & $\approx819$ \textcolor{lightgray}{\tiny ($204{,}800$ steps)}  & $1000$ & $100$   & $100$\\
Time embedding dim  & $32$   & $32$   & $64$    & $64$\\
Head                & velocity (MSE)
                    & velocity (MSE)
                    & \shortstack{endpoint\\(cross-entropy)}
                    & \shortstack{endpoint\\(cross-entropy)}\\
\bottomrule
\end{tabular}
\end{table}

\vspace{-3mm}\paragraph{Cycle-graph data}
Source and target are both distributions of cycle graphs on $N=12$ nodes.
A sample places twelve nodes at equispaced angles, offset by a rotation drawn
uniformly from $[0,2\pi)$, on the circle of radius $1$ centred at $(h,y)$, and
joins two nodes exactly when they are neighbours along that circle. Nodes
are then relabelled by a uniformly random permutation, so the node index carries
no information about the circular order. The tensor $E\in\R^{12\times12\times3}$
has adjacency in $\{0,1\}$ as its first channel, while the remaining two channels
store the planar node positions on the diagonal, $E_{ii,2:3}\in\R^2$.
Positions are stored divided by $8$, which puts the position
and adjacency channels on comparable scales in both the GW cost and the MSE
loss. The source takes $y=0$ and the target $y=8$, with
$h\sim\mathcal U[-6,6]$ drawn independently on each side. The two distributions
therefore agree up to a translation and have identical edge distributions.
To give the source a density, each training source is perturbed after coupling
by small symmetric Gaussian noise ($\sigma=0.01$) on all adjacency entries and
node positions. We draw $4000$ target graphs and a fresh source per pair, use
$\lambda_{\mathrm{edge}}=\lambda_{\mathrm{node}}=0.5$, and train for
$204{,}800$ steps at batch size $16$ with the velocity (MSE) head.
\textsc{GW}$_{\mathrm{out}}$ reorders chunks of $16$ pairs. For evaluation, we
integrate $125$ Euler steps from $1024$ exact source rings.

\vspace{-3mm}\paragraph{Continuous SBM data}
A graph on $N=10$ nodes with $K$ communities assigns the nodes to fixed
near-balanced blocks (remainder placed in the first blocks, e.g. $K=3\to4/3/3$)
in a randomized order. Parametrizing $\mathrm{Beta}(a,b)$ by its mean $\mu$ and
concentration $s$ via $a=s\mu$, $b=s(1-\mu)$, every unordered pair $i<k$ carries
an edge weight $\mathrm{Beta}(s_e\mu_{\mathrm{in}}, s_e(1-\mu_{\mathrm{in}}))$
when $i,k$ share a block and
$\mathrm{Beta}(s_e\mu_{\mathrm{out}}, s_e(1-\mu_{\mathrm{out}}))$ otherwise, with
$\mu_{\mathrm{in}}=0.75$, $\mu_{\mathrm{out}}=0.25$, and $s_e=8$ (i.e.
$\mathrm{Beta}(6,2)$ and $\mathrm{Beta}(2,6)$). The diagonal node feature of a
node in block $k$ is $\mathrm{Beta}(s_n\mu_k, s_n(1-\mu_k))$ with evenly spaced
means $\mu_k=(2k+1)/(2K)$ and $s_n=8$. For $K=2$ this reproduces the binary SBM
feature model exactly. In the unconditional variant graphs with $K\in\{1,\dots,5\}$
are pooled and in the conditional variant $K$ is the class label. In the unconditional variant, $2000$ graphs are drawn for each
$K\in\{1,\dots,5\}$ and pooled ($10\,000$ in total). In the conditional variant,
the same per-class sets are used, with $K$ supplied to the model as the class
label.

\vspace{-3mm}\paragraph{Molecular data}
Molecules are encoded from their RDKit representation after kekulization, so
bond orders are integer-valued. A molecule on $N$ heavy atoms becomes a tensor
$E\in\R^{N\times N\times C}$ whose edge channel is a one-hot over
$\{\text{none},\text{single},\text{double},\text{triple}\}$ and whose diagonal
node feature concatenates an atom-type one-hot with a formal-charge one-hot over
$\{-1,0,+1\}$. QM9 uses heavy atoms $N\le9$ over
\textcolor{black}{\(\{\mathrm C,\mathrm N,\mathrm O,\mathrm F\}\)}, giving $C=11$.
ZINC250k uses $N\le38$ over
\textcolor{black}{\(\{\mathrm C,\mathrm N,\mathrm O,\mathrm F,\mathrm{Br},
\mathrm{Cl},\mathrm I,\mathrm P,\mathrm S\}\)}, giving $C=16$. For
ZINC250k the $28$-way PyTorch~Geometric atom vocabulary---which encodes charged
species---is remapped to the nine element types listed above and to a signed
formal charge in $\{-1,0,+1\}$. $31.8\%$ of ZINC250k molecules carry a formal charge,
and modelling it explicitly is required for those atoms to pass valence on
decode. QM9 uses the GDSS train/test split, matched on raw QM9 record identifiers
(\texttt{valid\_idx\_qm9.json}). ZINC250k uses the standard PyTorch~Geometric
split with $220{,}011$ training, $24{,}445$ validation, and $5{,}000$ test
molecules. FCD is evaluated against the respective test split.

\vspace{-3mm}\paragraph{Variable node count}
For the molecular datasets, one set of graph-transformer parameters is shared
across all node counts. An input with $N$ nodes produces an output with $N$
nodes. Source noise for each target is drawn at the target's own node count, so
no padding or masking is used. During training, minibatches are formed by a
sampler that groups dataset indices by exact $N$ and yields same-$N$ batches.
This is required because the outer cost matrices of
Section~\ref{sec:graph-couplings} compare same-size tensors. In the conditional SBM variant, each batch
additionally contains only graphs with a single value of $K$, so graphs with
different community counts are never coupled. At generation time $N$ is sampled
from the empirical training-set node-count distribution and the Euler solver is
initialized from noise of shape $(N,N,C)$.
\textcolor{black}{Euler sampling evaluates the field at \(t_k=k/K\),
\(k=0,\ldots,K-1\), and steps to \(t=1\), so the network is not evaluated at
the singular endpoint. The final tensor is decoded by argmax within each
categorical block.}

\vspace{-3mm}\paragraph{Source noise}
For the molecular datasets the source is a symmetric Gaussian. For each edge
channel and each \(i<j\), we draw
\(
U_{ij},U_{ji}\overset{\mathrm{iid}}{\sim}
\mathcal N(0,\sigma_{\rm adj}^2)
\)
and set
\(
A_{ij}=A_{ji}=\frac{U_{ij}+U_{ji}}{\sqrt{2}},
\)
so that each off-diagonal edge entry satisfies
\(
A_{ij}\sim\mathcal N(0,\sigma_{\rm adj}^2).
\)
The edge diagonal is zero. Each diagonal node feature is drawn independently
from \(\mathcal N(0,\sigma_{\rm node}^2)\), with
\(\sigma_{\rm adj}=\sigma_{\rm node}=0.5\). For the continuous SBM the source
is independent \(\mathcal U[0,1]\) on both channels. The upper-triangular edge
entries are mirrored across the diagonal, and the diagonal node features are drawn
independently. This matches the \((0,1)\) scale of the dense Beta target. The edge
diagonal is again zero.

\paragraph{Compute scaling}
For each fixed node count ($N\in\{8,9\}$ for QM9, $N\in\{18,23\}$ for ZINC250k) we draw
$M=32{,}768$ source--target pairs once: targets are training graphs with $N$ nodes (with
replacement where fewer than $M$ exist), each uniformly relabeled, and sources are drawn from
the training source law at the same $N$. All methods couple this identical set of pairs. For
every budget $AB\in\{2^5,\dots,2^{15}\}$ and every power-of-two split $(A,B)$, the $M$ pairs are
partitioned into $M/B$ disjoint minibatches of $B$ pairs. Within each minibatch we evaluate
$\hat c_A(E^a,F^b;\Sigma_A)$ for all $B^2$ combinations, solve the outer assignment $\hat\tau$ by the
Hungarian algorithm, and keep for every assigned target its minimizing permutation $\hat\sigma_A$.
Thus $B=1$ is inner BoA$^A$ alone (pairing unchanged), $A=1$ is MBOT$^B$, $A=B=1$ is Random, and
$B=M$ is a single assignment over all pairs. Here $\Sigma_A$ consists of the identity and $A-1$
distinct further permutations drawn uniformly without replacement for QM9, so that at large $A$ no
budget is spent on repeated candidates (for QM9 with $N=8$, $A=2^{15}$ covers $81\%$ of $S_8$). For ZINC, we simply sample with replacement due to the negligible probability of sampling the same permutation twice.
Each point reports the mean aligned distance over all $M$ pairs and the single-threaded CPU
wall-clock time per coupled pair, including permutation sampling, cost evaluation and
assignment. All timings, including GW, are measured on the same machine. The simple GW reference applies the 10-iteration Frank--Wolfe solver
with Hungarian rounding to the same $M$ pairs with $\hat\tau=\mathrm{id}$. Unlike joint inner-outer GW alignment, its per-pair cost and
time do not depend on $B$. Outer GW assignment approximates the cost between two graphs via GW approximation with 5 iterations.

\vspace{-3mm}\paragraph{Coupling solvers}
All fused solvers use the combined graph entries from
Section~\ref{sec:graph-couplings}, whose squared norms are
\[
\lVert E_{ik}\rVert^2=\lambda_{\rm edge}\,\lVert e^E_{ik}\rVert^2
        +\tfrac{\lambda_{\rm node}}{C_{\rm v}}\,\lVert f^E_i\rVert^2\,\mathbf 1_{\{i=k\}},
\qquad \lambda_{\rm edge}=\lambda_{\rm node}=\tfrac12,
\]
where node features enter only on the diagonal. The inner Gromov--Wasserstein solver (\textsc{GW}) runs \(10\)
  Frank--Wolfe iterations, each linearized by an exact optimal-transport
  (Hungarian) step, and the final soft plan is projected to the
  Frobenius-nearest scaled permutation matrix using the Hungarian assignment in
  Section~\ref{sec:graph-couplings}. When GW is used for outer batch matching, we
  compute the pairwise batch-cost matrix with \(5\) Frank--Wolfe iterations before
  applying the Hungarian assignment over the batch.
  \textsc{BoA}$^A$ relabels each target uniformly at random and keeps, among this
labeling and $A-1$ further random permutations, the one with the smallest squared
Euclidean distance of the combined entries. For \textsc{BoA}$^A$+\textsc{BoA}$^B_{\rm out}$,
this best-of-$A$ cost is computed for all $B^2$ pairs of a same-$N$ pool, the Hungarian
assignment is solved on it, and each assigned target keeps its minimizing permutation
($A=1$ recovers \textsc{MBOT}$^B$). Pools larger than the gradient batch are assigned once
and then split into gradient batches, which leaves the number of steps per epoch unchanged.
During training, permutations are drawn i.i.d.\ with replacement: for $A\le 32$ the
expected fraction of repeated candidates is at most $(A-1)/(2\,N!)$, below $0.05\%$ for
$N\ge 8$.
The first lower bound (\textsc{FLB}) ranks
nodes by the root-mean-square eccentricity
$\operatorname{ecc}_E(i)=(N^{-1}\sum_k\lVert E_{ik}\rVert^2)^{1/2}$ defined in
Appendix~\ref{sbm:appendix_sec}. It sorts the nodes in each graph by this
value and pairs nodes at the same position in the two orderings. Computing the
eccentricities and sorting costs $O(N^2C+N\log N)$ per pair. For outer
reordering with FLB, each same-$N$ training batch is partitioned into sub-batches of at
most eight graphs like GW reordering. Within each sub-batch, we build the pairwise FLB cost matrix
and apply the Hungarian algorithm to permute the target graphs.

\vspace{-3mm}\paragraph{Evaluation metrics}
For cycle graphs, roundness is $\sigma_r=\operatorname{std}_i\|p_i-\bar p\|$, the standard deviation of the node distances to their centroid $\bar p=\frac1N\sum_i p_i$ (reported $\times10^2$, $0$ for an exact circle). The positions $p_i$  are in unscaled plot units, i.e. the stored values multiplied back by 8. Mean degree is $\frac1N\sum_{i\neq j}A_{ij}$ on the raw, symmetrized continuous adjacency, which equals $2$ for an exact $12$-cycle.
It does not round the adjacency, so partially present edges count with their fractional value.
 For molecules, validity is the fraction of generated
graphs that decode to a sanitizable RDKit molecule without valence correction
or other postprocessing. For FCD, we retain the largest connected fragment of each sanitizable generated
  molecule, canonicalize it, remove duplicates, and compare the ChemNet embeddings
  of these unique generated SMILES with the corresponding unique test-split SMILES.
  The descriptor MMDs follow the GraphRNN/GDSS conventions
\citep{jo2022gdss,you2018graphrnn} with the biased estimator and kernels of the
form $k(x,y)=\exp[-d(x,y)^2/(2\sigma^2)]$. Degree uses earth mover's distance
(EMD) with $\sigma=1.0$ on per-graph degree histograms. Clustering uses EMD over
$100$ bins on $[0,1]$ with $\sigma=10$ in raw-bin units (equivalently
$\sigma=0.1$ in the $[0,1]$-normalized EMD units of GraphRNN/GDSS). Graphlet
orbit uses Euclidean distance in a Gaussian radial basis function kernel with
$\sigma=30.0$ on the $15$-dimensional, per-node-averaged four-node
graphlet-orbit counts computed with the ORCA graphlet-counting program
\citep{hocevar2014combinatorial}. FGW--NNA pools the $n$ generated and $n$ real
graphs and classifies each by the label of its nearest neighbor (self excluded)
under the fused Gromov--Wasserstein distance with
$\lambda_{\rm edge}=\lambda_{\rm node}=\tfrac12$. Under the null this
accuracy tends to $1/2$.

\clearpage
\vspace{-3.5mm}\section{Training convergence of the molecular sweep}
\label{app:convergence-epochs}

Table~\ref{tab:mol-sweep_ODE} reports only the final checkpoints. We additionally
evaluate the EMA checkpoints every $25$ epochs, using $5$ sets of $1024$
molecules at $125$ Euler steps. Seeds are matched across couplings and epochs.
We use $125$ steps to reduce integration error and isolate training progress.
The epoch-$100$ results agree with Table~\ref{tab:mol-sweep_ODE} up to the
smaller evaluation sample.

\begin{figure}[h]
    \centering
    \begin{subfigure}[t]{0.49\textwidth}
        \centering
        \includegraphics[width=\textwidth]{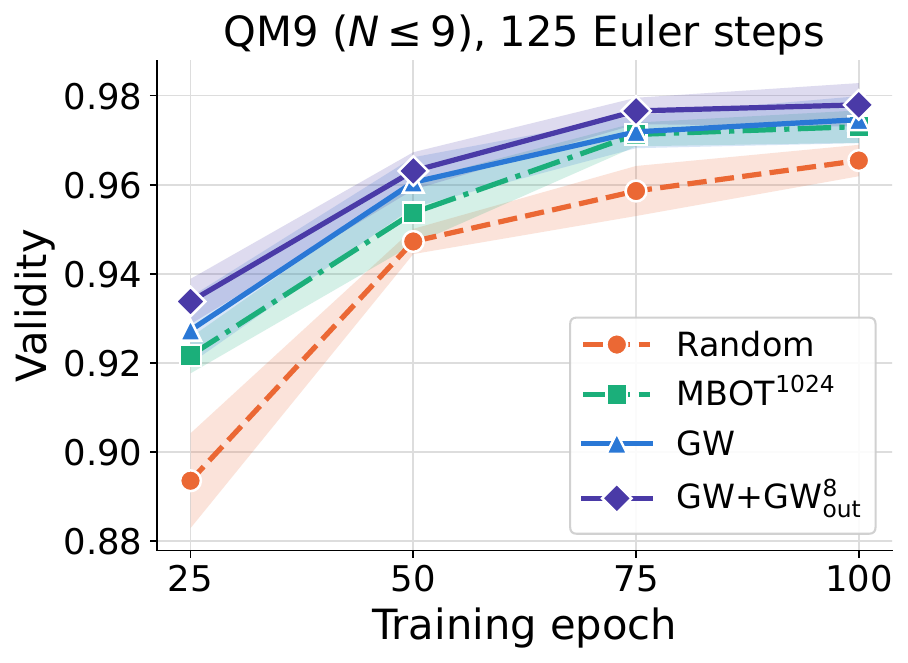}
        \caption{QM9 ($N\le9$)}
        \label{subfig:convergence_qm9}
    \end{subfigure}
    \hfill
    \begin{subfigure}[t]{0.49\textwidth}
        \centering
        \includegraphics[width=\textwidth]{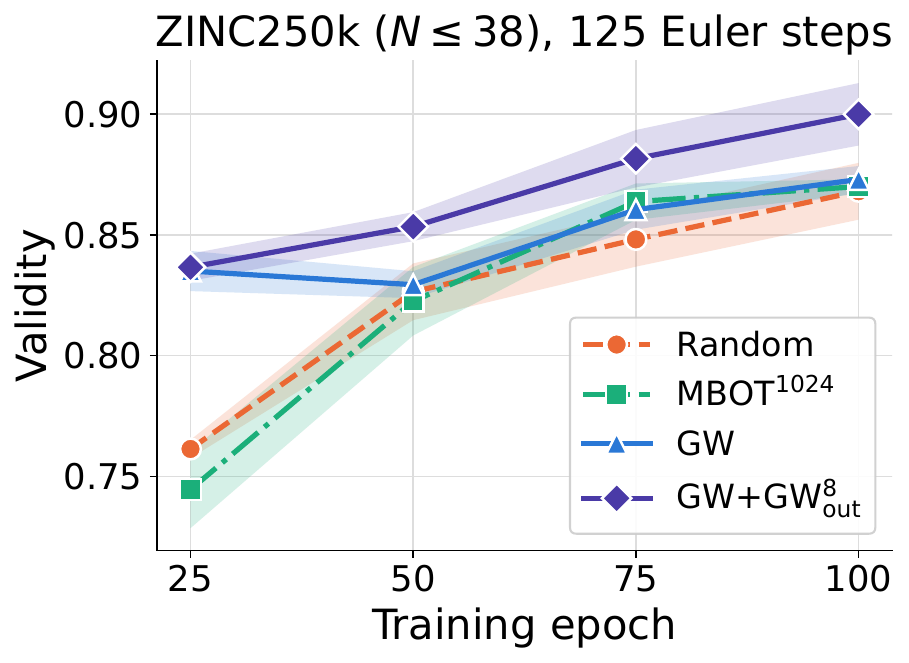}
        \caption{ZINC250k ($N\le38$)}
        \label{subfig:convergence_zinc}
    \end{subfigure}
    \caption{Validity during training at $125$ Euler steps, averaged over
    $5$ repeats of $1024$ molecules. The shaded bands show $\pm1$ standard deviation.}
    \label{fig:convergence_epochs}
\end{figure}

Figure~\ref{fig:convergence_epochs} shows that GW alignment accelerates early training:
validity differences between couplings are largest at the first evaluated checkpoint
and much smaller by the end of training. On QM9, the
\textsc{GW}$+$\textsc{GW}$_{\mathrm{out}}$ advantage over \textsc{Random} likewise shrinks
substantially. This hints at faster convergence with aligned couplings, presumably because
consistent node correspondences reduce the variability in the loss.

\clearpage\section{Uniqueness}
\label{sec:mol-sweep_uniqueness}
Uniqueness is the fraction of distinct largest-fragment canonical SMILES among the valid generated molecules, which we can again evaluate on $10{,}000$ generated samples per repeat, averaged over $3$ repeats.
However, uniqueness (Table \ref{tab:mol-sweep_uniqueness}) does not separate the couplings. 
Essentially, it is saturated on ZINC250k, and on QM9 the only visible effect is a slight GW advantage at the lowest Euler step count that vanishes at higher step counts and is within noise for the flow maps. Because uniqueness only counts distinct valid SMILES, it ignores whether a sample is chemically sensible, unlike validity or FCD, so it cannot serve as evidence of sample quality. What it does show is that the gains of GW alignment in validity and FCD are not bought at the expense of diversity.
\begin{table}[H]
  \centering
  \scriptsize
  \setlength{\tabcolsep}{1pt}
  \caption{Molecular short-budget sweep, uniqueness (same runs and evaluations as Table~\ref{tab:mol-sweep}): (\ref{tab:mol-sweep_ODE_uniqueness}) Flow ODE at
  $5/25/125$ Euler steps and (\ref{tab:mol-sweep_flow_map_uniqueness}) flow map at $1/2/5$
  steps. Rows are mean$_{\text{std}}$ over $3$ evaluation repeats. Best is
  bold (up to $\pm2\sigma$ of evaluation variability).}
  \vspace{-2mm}
  \label{tab:mol-sweep_uniqueness}
  \begin{subtable}[t]{0.49\textwidth}
  \centering
  \caption{\scriptsize Euler Integration of Learned Flow ODE}
  \label{tab:mol-sweep_ODE_uniqueness}
  \begin{adjustbox}{max width=\linewidth}
  \begin{tabular}{l ccc}
  \toprule
  & \multicolumn{3}{c}{Uniqueness ($\uparrow$)}\\
  \cmidrule(lr){2-4}
  Method & $5$ & $25$ & $125$\\
  \midrule
  \multicolumn{4}{l}{\emph{QM9} ($N\le9$)}\\
  \textsc{Random}
   & $0.9342_{.0036}$ & $\mathbf{0.9678}_{.0027}$ & $\mathbf{0.9713}_{.0010}$\\
  \textsc{MBOT}$^{1024}$
   & $0.9426_{.0028}$ & $0.9647_{.0011}$ & $0.9661_{.0007}$\\
  \textsc{BoA}$^{4}$$+$\textsc{BoA}$^{256}_{\mathrm{out}}$
   & $0.9464_{.0019}$ & $\mathbf{0.9668}_{.0016}$ & $0.9679_{.0007}$\\
  \textsc{GW}
   & $0.9531_{.0009}$ & $0.9650_{.0010}$ & $0.9685_{.0006}$\\
  \textsc{GW}$+$\textsc{GW}$^8_{\mathrm{out}}$
   & $\mathbf{0.9585}_{.0022}$ & $\mathbf{0.9684}_{.0007}$ & $0.9685_{.0005}$\\
  \midrule
  \multicolumn{4}{l}{\emph{ZINC250k} ($N\le38$)}\\
  \textsc{Random}
   & $\mathbf{0.9999}_{.0001}$ & $\mathbf{1.0000}_{.0001}$ & $\mathbf{1.0000}_{.0001}$\\
  \textsc{MBOT}$^{1024}$
   & $\mathbf{0.9998}_{.0002}$ & $\mathbf{1.0000}_{.0000}$ & $\mathbf{0.9999}_{.0001}$\\
  \textsc{BoA}$^{2}$$+$\textsc{BoA}$^{512}_{\mathrm{out}}$
   & $\mathbf{0.9999}_{.0001}$ & $\mathbf{1.0000}_{.0001}$ & $\mathbf{0.9999}_{.0001}$\\
  \textsc{GW}
   & $\mathbf{1.0000}_{.0000}$ & $\mathbf{1.0000}_{.0001}$ & $\mathbf{0.9999}_{.0001}$\\
  \textsc{GW}$+$\textsc{GW}$^8_{\mathrm{out}}$
   & $\mathbf{1.0000}_{.0000}$ & $0.9997_{.0001}$ & $\mathbf{0.9999}_{.0000}$\\
  \bottomrule
  \end{tabular}
  \end{adjustbox}
  \end{subtable}
  \hfill
  \begin{subtable}[t]{0.49\textwidth}
  \centering
  \caption{\scriptsize Distilled Flow Map for Few-Step Generation}
  \label{tab:mol-sweep_flow_map_uniqueness}
  \begin{adjustbox}{max width=\linewidth}
  \begin{tabular}{l ccc}
  \toprule
  & \multicolumn{3}{c}{Uniqueness ($\uparrow$)}\\
  \cmidrule(lr){2-4}
  Method & $1$ & $2$ & $5$\\
  \midrule
  \multicolumn{4}{l}{\emph{QM9} ($N\le9$)}\\
  \textsc{Random}
   & $\mathbf{0.9641}_{.0015}$ & $\mathbf{0.9686}_{.0011}$ & $\mathbf{0.9702}_{.0015}$\\
  \textsc{MBOT}$^{1024}$
   & $0.9613_{.0009}$ & $0.9620_{.0011}$ & $\mathbf{0.9662}_{.0022}$\\
  \textsc{BoA}$^{4}$$+$\textsc{BoA}$^{256}_{\mathrm{out}}$
   & $\mathbf{0.9643}_{.0015}$ & $0.9624_{.0019}$ & $\mathbf{0.9689}_{.0007}$\\
  \textsc{GW}
   & $\mathbf{0.9663}_{.0020}$ & $\mathbf{0.9667}_{.0013}$ & $\mathbf{0.9685}_{.0011}$\\
  \textsc{GW}$+$\textsc{GW}$^8_{\mathrm{out}}$
   & $\mathbf{0.9639}_{.0010}$ & $\mathbf{0.9650}_{.0019}$ & $\mathbf{0.9710}_{.0014}$\\
  \midrule
  \multicolumn{4}{l}{\emph{ZINC250k} ($N\le38$)}\\
  \textsc{Random}
   & $\mathbf{1.0000}_{.0000}$ & $\mathbf{0.9999}_{.0001}$ & $\mathbf{0.9999}_{.0001}$\\
  \textsc{MBOT}$^{1024}$
   & $\mathbf{0.9999}_{.0001}$ & $\mathbf{1.0000}_{.0000}$ & $\mathbf{1.0000}_{.0000}$\\
  \textsc{BoA}$^{2}$$+$\textsc{BoA}$^{512}_{\mathrm{out}}$
   & $\mathbf{1.0000}_{.0000}$ & $\mathbf{1.0000}_{.0000}$ & $\mathbf{1.0000}_{.0001}$\\
  \textsc{GW}
   & $\mathbf{1.0000}_{.0000}$ & $\mathbf{1.0000}_{.0000}$ & $\mathbf{0.9999}_{.0001}$\\
  \textsc{GW}$+$\textsc{GW}$^8_{\mathrm{out}}$
   & $\mathbf{1.0000}_{.0000}$ & $\mathbf{0.9999}_{.0001}$ & $\mathbf{0.9997}_{.0003}$\\
  \bottomrule
  \end{tabular}
  \end{adjustbox}
  \end{subtable}
\end{table}

\clearpage\section{Solver cost and correlation with GW}
\label{app:solver-cost}

Table~\ref{tab:solver-cost} compares \textsc{GW} ($10$ Frank--Wolfe iterations) with
cheaper pair costs on the pairs seen in training. These are \textsc{FLB}, the Euclidean
distance $\|E-F\|^2$ used by \textsc{MBOT}, and \textsc{BoA}$^{\mathcal A}$ with
$\mathcal A=32$ and $\mathcal A=1024$. For each dataset we draw
$300$ pairs of a source sample and a training graph with fixed $N$ ($N=10$ for the
continuous SBM, $N=9$ for QM9, $N=23$ for ZINC250k). Sources follow the training source
law, uniform for the continuous SBM and Gaussian for the molecules. We set
$\lambda_{\rm edge}=\lambda_{\rm node}=\tfrac12$.

\begin{table}[h]
\centering\footnotesize\setlength{\tabcolsep}{5pt}
\caption{Runtime in ms per pair (CPU, single thread) and correlation with the
\textsc{GW} cost, Pearson (Spearman), over 300 pairs of a training source sample and a
same-$N$ training graph. Euclidean is the \textsc{MBOT} cost $\|E-F\|^2$ and
\textsc{BoA}$^{\mathcal A}$ its minimum over $\mathcal A$ random relabellings.}
\label{tab:solver-cost}
\begin{tabular}{l cc cc cc}
\toprule
& \multicolumn{2}{c}{Continuous SBM} & \multicolumn{2}{c}{QM9} & \multicolumn{2}{c}{ZINC250k}\\
& \multicolumn{2}{c}{\scriptsize $N=10$} & \multicolumn{2}{c}{\scriptsize $N=9$} & \multicolumn{2}{c}{\scriptsize $N=23$}\\
\cmidrule(lr){2-3}\cmidrule(lr){4-5}\cmidrule(lr){6-7}
Method & ms/pair & Corr.\ to \textsc{GW} & ms/pair & Corr.\ to \textsc{GW} & ms/pair & Corr.\ to \textsc{GW}\\
\midrule
\textsc{GW}
 & $1.30$ & --- & $2.21$ & --- & $2.66$ & ---\\
\textsc{FLB}
 & $0.053$ & $0.64$ ($0.56$) & $0.061$ & $-0.38$ ($-0.34$) & $0.065$ & $-0.45$ ($-0.46$)\\
Euclidean
 & $0.004$ & $0.42$ ($0.38$) & $0.004$ & $0.77$ ($0.77$) & $0.006$ & $0.90$ ($0.89$)\\
\textsc{BoA}$^{32}$
 & $0.036$ & $0.72$ ($0.67$) & $0.043$ & $0.91$ ($0.91$) & $0.16$ & $0.95$ ($0.94$)\\
\textsc{BoA}$^{1024}$
 & $0.44$ & $0.80$ ($0.74$) & $0.60$ & $0.94$ ($0.94$) & $5.61$ & $0.96$ ($0.96$)\\
\bottomrule
\end{tabular}
\end{table}

\textsc{FLB} is $24$ to $41$ times cheaper than \textsc{GW}, and the Euclidean distance
more than $100$ times. \textsc{BoA}$^{\mathcal A}$ evaluates $\mathcal A$ Euclidean
distances, so its cost grows linearly in $\mathcal A$. On all datasets
\textsc{BoA}$^{32}$ correlates better with \textsc{GW} than the Euclidean distance, and
\textsc{BoA}$^{1024}$ better still, but it is slower than \textsc{GW} on
ZINC250k. The gain from relabelling is largest on the continuous SBM. \textsc{FLB} is
negatively correlated with \textsc{GW} on molecules, because the one-hot targets have
constant eccentricities, so \textsc{FLB} only measures the source. See
\citep{piening2025novel} for further lower bounds.

Table~\ref{tab:outer-cost} reports the cost per training batch. Inner \textsc{GW} is
$48$ to $78$ times as expensive as \textsc{MBOT}, and outer \textsc{GW} multiplies
this by about $5.4$.
\textsc{BoA}$^{\mathcal A}+$\textsc{BoA}$^{B}_{\mathrm{out}}$ evaluates
$B^2\mathcal A$ distances and is the slowest non-\textsc{GW} method on molecules, but
still $14$ to $133$ times cheaper than \textsc{GW}$+$\textsc{GW}$^{8}_{\mathrm{out}}$.

\begin{table}[h]
\centering\footnotesize\setlength{\tabcolsep}{5pt}
\caption{Alignment cost per training batch in ms (CPU, single thread). Outer \textsc{FLB} and \textsc{GW} reorder in
sub-batches of eight. \textsc{MBOT} and \textsc{BoA}$_{\mathrm{out}}$ assign over the full
minibatch.}
\label{tab:outer-cost}
\begin{tabular}{l ccc}
\toprule
 & Continuous SBM & QM9 & ZINC250k\\
 & {\scriptsize $N=10$, $B=16$, $\mathcal A=32$} & {\scriptsize $N=9$, $B=64$, $\mathcal A=32$} & {\scriptsize $N=23$, $B=32$, $\mathcal A=32$}\\
\midrule
\textsc{MBOT}$^{B}$
 & $0.30$ & $1.74$ & $1.71$\\
\textsc{FLB}
 & $0.75$ & $3.39$ & $1.94$\\
\textsc{FLB}$+$\textsc{FLB}$^{8}_{\mathrm{out}}$
 & $1.48$ & $6.54$ & $3.68$\\
\textsc{BoA}$^{\mathcal A}$
 & $0.51$ & $2.62$ & $5.07$\\
\textsc{BoA}$^{\mathcal A}+$\textsc{BoA}$^{B}_{\mathrm{out}}$
 & $0.82$ & $9.82$ & $31.14$\\
\textsc{GW}
 & $20.21$ & $136.27$ & $82.76$\\
\textsc{GW}$+$\textsc{GW}$^{8}_{\mathrm{out}}$
 & $109.50$ & $734.24$ & $442.41$\\
\bottomrule
\end{tabular}
\end{table}

\end{document}